\documentclass[10pt,journal,compsoc]{IEEEtran}
\ifCLASSOPTIONcompsoc
  \usepackage[nocompress]{cite}
\else
  \usepackage{cite}
\fi
\ifCLASSINFOpdf
   \usepackage[pdftex]{graphicx}
   \graphicspath{{./figures/}{../jpeg/}}
\else
   \usepackage[dvips]{graphicx}
\fi
\usepackage{amsmath}
\usepackage{graphicx}
\usepackage{multirow}
\usepackage{gensymb}
\usepackage{amsfonts}
\usepackage{booktabs}
\usepackage{amsthm}
\usepackage{amssymb}
\usepackage{mathtools}
\usepackage{bm}

\newtheorem{definition}{Definition}
\newtheorem{theorem}{Theorem}

\newtheorem{remark}{Remark}
\newtheorem{assumption}{Assumption}
\usepackage{color}
\definecolor{neg}{RGB}{248, 206, 204}
\definecolor{pos}{RGB}{213, 232, 212}

\begin{document}

\title{Deep Safe Multi-Task Learning}

\author{Zhixiong Yue$^*$, Feiyang Ye$^*$, Yu Zhang, Christy Liang, and Ivor W. Tsang
\IEEEcompsocitemizethanks{\IEEEcompsocthanksitem Z. Yue and F. Ye are with Department of Computer Science and Engineering, Southern University of Science and Technology and School of Computer Science, University of Technology Sydney. Y. Zhang is with Department of Computer Science and Engineering, Southern University of Science and Technology and Peng Cheng Laboratory. C. Liang is with School of Computer Science, University of Technology Sydney. Ivor W. Tsang is with Australian Artificial Intelligence Institute, University of Technology Sydney. 
\IEEEcompsocthanksitem $^*$ Equal contribution.
\IEEEcompsocthanksitem Corresponding author: Yu Zhang (yu.zhang.ust@gmail.com).}
}
\markboth{}%
{}

\IEEEtitleabstractindextext{%
\begin{abstract}
In recent years, Multi-Task Learning (MTL) has attracted much attention due to its good performance in many applications. However, many existing MTL models cannot guarantee that their performance is no worse than their single-task counterparts on each task. Though some works have empirically observed this phenomenon, little work aims to handle the resulting problem. In this paper, we formally define this phenomenon as \textit{negative sharing} and define \textit{safe multi-task learning} where no \textit{negative sharing} occurs. To achieve \textit{safe multi-task learning}, we propose a Deep Safe Multi-Task Learning (DSMTL) model with two learning strategies: individual learning and joint learning. We theoretically study the safeness of both learning strategies in the DSMTL model to show that the proposed methods can achieve some versions of \textit{safe multi-task learning}. Moreover, to improve the scalability of the DSMTL model, we propose an extension, which automatically learns a compact architecture and empirically achieves \textit{safe multi-task learning}.
Extensive experiments on benchmark datasets verify the safeness of the proposed methods.
\end{abstract}

}

\maketitle

\IEEEpeerreviewmaketitle

\IEEEraisesectionheading{\section{Introduction}\label{sec:introduction}}

\IEEEPARstart{M}{ulti-Task} Learning (MTL) \cite{caruana97,zy21}, which aims to improve the generalization performance of multiple learning tasks by sharing knowledge among those tasks, has attracted much attention in recent years. Compared with single-task learning that learns each task independently, MTL not only improves the performance for some or all the tasks but also reduces the training and inference time. Therefore, MTL has been widely used in many Computer Vision (CV) applications, such as human action recognition \cite{liu2016hierarchical}, face attribute estimation \cite{han2017heterogeneous}, age estimation \cite{zhang2010multi}, and dense prediction tasks \cite{vandenhende2021multi}.

Although MTL has demonstrated its usefulness in many applications, MTL cannot guarantee to improve the performance of all the tasks compared with single-task learning. Specifically, as empirically observed in \cite{lee2016asymmetric, guo2020learning, SunPFS20,standley2020tasks, tang2020progressive}, when learning multiple tasks together,
many existing MTL models can achieve better performance on some tasks than their single-task counterparts but underperform on the other tasks. Such phenomenon is called the \textit{negative sharing} phenomenon in this paper, which is similar to the `negative transfer' phenomenon \cite{wang2019characterizing} in transfer learning \cite{yang2020transfer} but with some differences
as discussed later. One reason for the occurrence of \textit{negative sharing} is that there are partially related or even unrelated tasks among tasks under the investigation, making jointly learning those tasks impair the performance of some tasks.

To the best of our knowledge, there is little work to study the \textit{negative sharing} problem for MTL. To fill this gap, in this paper, we firstly give a formal definition for \textit{negative sharing} that could occur in MTL. Then we formally define an ideal and also basic situation for MTL called \textit{safe multi-task learning}, where the generalization performance of an MTL model is no worse than its single-task counterpart on each task. That is, there is no \textit{negative sharing} occurred. According to the definition of MTL \cite{caruana97,zy21}, we can see that every MTL model is required to achieve \emph{safe multi-task learning}. Otherwise, single-task learning is more preferred than MTL, since an unsafe MTL model may bring the risk of worsening the generalization performance of some or even all the tasks. As true data distributions in multiple tasks are usually unknown so that \textit{safe multi-task learning} is hardly to measure, we formally define \textit{empirically safe multi-task learning} and \textit{probably safe multi-task learning}, which are measurable.

To achieve \textit{empirically/probably safe multi-task learning}, we propose a Deep Safe Multi-Task Learning (DSMTL) model whose architecture consists of a public encoder shared by all the tasks and a private encoder for each task. The public encoder and a private encoder of a task are combined via a gating mechanism to form the entire encoder for that task. To train the DSMTL model, we propose two learning strategies: individual learning (denoted by DSMTL-IL) and joint learning (denoted by DSMTL-JL), which learn model parameters separately and jointly, respectively. For those two strategies, we provide theoretical analyses to show that they can achieve some versions of both \textit{empirically safe multi-task learning} and \textit{probably safe multi-task learning}.

To improve the scalability of the DSMTL model with respect to the number of tasks, we propose an extension called DSMTL with Architecture Learning (DSMTL-AL),
which leverages neural architecture search to learn a more compact architecture with fine-grained modular splitting. Specifically, we allow the DSMTL-AL model to learn where to switch to the private encoder while forwarding in the public encoder. In this way, the DSMTL-AL model can save the first few modules in the private encoders and hence improve the scalability.

Extensive experiments on benchmark datasets, including CityScapes, NYUv2, PASCAL-Context, and Taskonomy,
demonstrate the effectiveness of the proposed DSMTL-IL, DSMTL-JL, and DSMTL-AL methods.

The main contributions of this paper are summarized as follows.
\begin{itemize}
\item We provide formal definitions for MTL, including \textit{negative sharing}, \textit{safe multi-task learning}, \textit{empirically safe multi-task learning}, and \textit{probably safe multi-task learning}.

\item We propose the simple and effective DSMTL model with two learning strategies, which is guaranteed to achieve some versions of \textit{empirically/probably safe multi-task learning}.

\item We propose the DSMTL-AL method, which is an extension of the DSMTL methods, to learn a compact architecture with good scalability.

\item Extensive experiments demonstrate that empirically the proposed methods can achieve \emph{safe multi-task learning} and that they outperform state-of-the-art baseline models.
\end{itemize}

\section{Related Work}

MTL has been extensively studied in recent years \cite{evgeniou2004regularized, zhang2010convex, kumar2012learning, zhang2021multi, guo2021deep}. How to design a good network architecture for MTL is an important issue. The most widely used architecture is the multi-head hard sharing architecture \cite{caruana97, lmzcl15, ruder2019latent}, which shares the first several layers among all the tasks and allows the subsequent layers to be specific to different tasks. Then, to better handle task relationships, different MTL architectures have been proposed. For example, \cite{misra2016cross} proposes a cross-stitch network to learn to linearly combine hidden representations of different tasks.
\cite{ma2018modeling} proposes a multi-gate mixture-of-experts model which adopts the mixture-of-experts model by sharing expert submodels across all tasks, while having a gating network trained to optimize each task. \cite{liu2019end} proposes a Multi-Task Attention Network (MTAN), which consists of a shared network and an attention module for each task so that both shared and private feature representations can be learned via the attention mechanism. \cite{gao2019nddr} proposes a Neural Discriminative Dimensionality Reduction (NDDR) layer to enable automatic feature fusing at every layer for different tasks. \cite{SunPFS20} proposes an Adaptive Sharing (AdaShare) method to learn the sharing pattern through a policy that selectively chooses which layers to be executed for each task. \cite{cui2021adaptive} proposes an Adaptive Feature Aggregation (AFA) layer, where a dynamic aggregation mechanism is designed to allow each task to adaptively determine the degree of the knowledge sharing between tasks.
PS-MCNN \cite{cao2018partially} adopts both shared network and task-specific network by performing the concatenation operation after each block to learn shared and task-specific representations.
PLE \cite{tang2020progressive} separates shared components and task-specific components explicitly and adopts a progressive routing mechanism to extract semantic knowledge gradually for MTL.
Some routing-based methods are proposed, including Multi-Agent Reinforcement Learning (MARL) \cite{rosenbaum2018routing} that allows the MTL network to dynamically self-organize its architecture in response to the input, Stochastic Filter Groups (SFG) \cite{bragman2019stochastic} that assigns convolution kernels in each layer to the ``specialist'' or ``generalist'' group, and Task Routing Layer (TRL) \cite{strezoski2019many} that allows for a single model to fit to many tasks within its parameter space with task-specific masking.

Instead of hand-crafting architectures for MTL, there are some works to leverage techniques in Neural Architecture Search (NAS) \cite{DARTS} to automatically search MTL architectures with good performance.
For example, \cite{lu2017fully} dynamically widens a multi-layer network to create a tree-like deep architecture, where similar tasks reside in the same branch.
\cite{EASMTL} proposes an evolutionary architecture search algorithm to search blueprints and modules that are assembled into an MTL network.
\cite{MTL-NAS} searches inter-task layers for better feature fusion across tasks.
\cite{LTB} proposes a differentiable architecture search algorithm to learn branching blocks to construct a tree-structured neural network for MTL.
\cite{BMTAS} automatically determines the branching architecture for the encoder in a multi-task neural network under resource constraints.
\cite{guo2020learning} aims to learn where to share or branch within a network for multiple tasks.
\cite{sun2021task} designs a task switching network that can learn to switch between tasks with a constant number of parameters which is independent of the number of tasks.
All the aforementioned works do not study how to achieve \textit{safe multi-task learning}, which is the focus of this paper.

The safeness of machine learning methods has drawn attention in recent years \cite{li2014towards,li2019towards,guo2020safe,tao2018reliable,tang2022deep,tang2022deepsafe}.
\cite{li2014towards} proposes a safe semi-supervised support vector machine that performs no worse than the supervised counterpart, leading to the safeness in the use of unlabeled data.
\cite{li2019towards} addresses the safe weakly supervised learning problem by integrating multiple weakly supervised learners, which is guaranteed to derive a safe prediction under a mild condition.
\cite{guo2020safe} proposes a safe deep semi-supervised learning method to alleviate the harm caused by class distribution mismatch.
Moreover, there are some works \cite{tao2018reliable,tang2022deep,tang2022deepsafe} to address the safeness in multi-view clustering. \cite{tao2018reliable} proposes reliable multi-view clustering, which empirically performs no worse than its single-view counterpart and proves that its performance will not significantly degrade under some assumptions. \cite{tang2022deep} proposes deep safe multi-view clustering to reduce the risk of performance degradation caused by view increasing and hence to guarantee to achieve the safeness in multi-view clustering. \cite{tang2022deepsafe} proposes a bi-level optimization framework to achieve safe incomplete multi-view clustering.

Different from the aforementioned works that address the safeness in semi-supervised learning, weakly supervised learning, and multi-view clustering, our work focuses on the safeness in multi-task learning.

\section{Definitions}

In this section, we formally introduce some definitions to measure the safeness in MTL.

We first define the \emph{negative sharing} phenomena.

\begin{definition}[Negative Sharing] \label{definition_negative_sharing}
For an MTL model which is trained on multiple learning tasks jointly, if its generalization performance on some tasks is inferior to the generalization performance of the corresponding single-task counterpart that is trained on each task separately, then \emph{negative sharing} occurs.
\end{definition}

\begin{remark}
\emph{Negative sharing} could occur when some tasks are partially or totally unrelated to other tasks. In this case, manually enforcing all the tasks to have some forms of sharing will impair the performance of some or even all the tasks. In Definition \ref{definition_negative_sharing}, the MTL model and its single-task counterpart usually have similar architectures, since totally different architectures could bring additional confounding factors. Moreover, \emph{negative sharing} is similar to negative transfer \cite{wang2019characterizing} in transfer learning \cite{yang2020transfer}. However, knowledge transfer in transfer learning is directed as it is from a source domain to a target domain, while knowledge sharing in MTL is among all the tasks, making it usually undirected. From this perspective, \emph{negative sharing} is different from negative transfer.
\end{remark}

\begin{definition}[Safe Multi-Task Learning] \label{definition_safe_multi-task_learning}
When \emph{negative sharing} does not occur for an MTL model on a dataset, this MTL model is said to achieve \emph{safe multi-task learning} on this dataset.
\end{definition}


\begin{remark}

\emph{Safe multi-task learning} is an ideal situation for an MTL model to achieve.
However, the generalization performance is hard to evaluate during the learning process, so it is hard to determine whether an MTL model can achieve \emph{safe multi-task learning}.

\end{remark}

As the empirical/training loss is easy to compute during the learning process, we present the following definition based on the empirical loss to measure the empirical safeness of MTL models.

\begin{definition}[Empirically Safe Multi-Task Learning] \label{definition_empirically_safe_multi-task_learning} \leavevmode \\
$(1)$ If the empirical loss of an MTL model on each task is no larger than that of its single-task counterpart, this MTL model is said to achieve \emph{empirically individual safe multi-task learning}.
\leavevmode \\ $(2)$  If the average of empirical losses of an MTL model on all the tasks is no larger than that of its single-task counterpart, this MTL model is said to achieve \emph{empirically average safe multi-task learning}.
\end{definition}

\begin{remark}

In Definition \ref{definition_empirically_safe_multi-task_learning}, we define two versions of \emph{empirically safe multi-task learning}. It is easy to see that an MTL model satisfying \emph{empirically individual safe multi-task learning} can achieve \emph{empirically average safe multi-task learning} but not vice versa, which indicates that \emph{empirically individual safe multi-task learning} is weaker than \emph{empirically individual safe multi-task learning}.
Even though \emph{empirically safe multi-task learning} is easy to measure based on the empirical loss of each task, an MTL model that achieves \emph{empirically safe multi-task learning} cannot have guarantee to achieve \emph{safe multi-task learning}, since there is a gap between the empirical loss and the expected loss that is to measure the generalization performance. Hence, \emph{empirically safe multi-task learning} is a loose version of \emph{safe multi-task learning}.

\end{remark}

With $m$ learning tasks in an MTL problem, $\mathcal{E}_t$ and $\mathcal{E}^{\text{STL}}_t$ denote the expected losses of an MTL model and its single-task counterpart on task $t$, respectively. The corresponding average expected losses are denoted by $\mathcal{E}$ and $\mathcal{E}^{\text{STL}}$, respectively. $n$ denotes the average number of samples in all the tasks.
With the above notations, we can define \emph{probably safe multi-task learning} as follows.

\begin{definition}[Probably Safe Multi-Task Learning] \label{definition_expected_safe_multi-task_learning}\leavevmode\\
$(1)$ For an MTL model trained on $m$ tasks, if for $0< \delta < 1$, there exist $m$ constants $\epsilon_t \ge 0$ such that $\mathcal{E}_t + \epsilon_t \le \mathcal{E}^{\text{STL}}_t +\rho^t_n$ holds with at least probability $1-\delta$ for any $t\in[1,m]$, where $\rho^t_n$ is a function of $n$ satisfying $\lim_{n\to +\infty}\rho^t_n = 0$, then this MTL model is said to achieve \emph{probably individual safe multi-task learning}.\\
$(2)$ If for $0< \delta < 1$, there exists a constant $\epsilon \ge 0$ such that $\mathcal{E} + \epsilon \le \mathcal{E}^{\text{STL}} +\rho_n$ holds with at least probability $1-\delta$, where $\rho_n$ is a function of $n$ satisfying $\lim_{n\to +\infty}\rho_n = 0$, then this MTL model is said to achieve \emph{probably average safe multi-task learning}.
\end{definition}

\begin{remark}

Different from \emph{empirically safe multi-task learning} which is measured based on empirical losses, \emph{probably safe multi-task learning} is based on expected losses, making it a tighter approximation of \emph{safe multi-task learning} than \emph{empirically safe multi-task learning}.
Compared with \emph{safe multi-task learning}, \emph{probably safe multi-task learning} is easy to be measured based on some analysis tool as verified in Section \ref{sec:analysis}.
According to Definition \ref{definition_expected_safe_multi-task_learning}, it is easy to see when $n$ is large enough, \emph{probably individual safe multi-task learning} could become \emph{safe multi-task learning} in a large probability.
Between the two versions of \emph{probably safe multi-task learning}, similar to \emph{empirically safe multi-task learning},  \emph{probably average safe multi-task learning} is weaker.

\end{remark}

As discussed above,  \emph{empirically safe multi-task learning} and \emph{probably safe multi-task learning} are two measures for an MTL model to achieve but few works can guarantee to achieve that. In the next section, we will propose the DSMTL model with such guarantees under mild conditions.




\section{DSMTL}


In this section, we present the proposed DSMTL model. Beside the network architecture, we introduce two strategies to learn model parameters, leading to two variants (i.e., DSMTL-IL and DSMTL-JL).

\begin{figure}[h]
\centering
\includegraphics[width=0.9\linewidth]{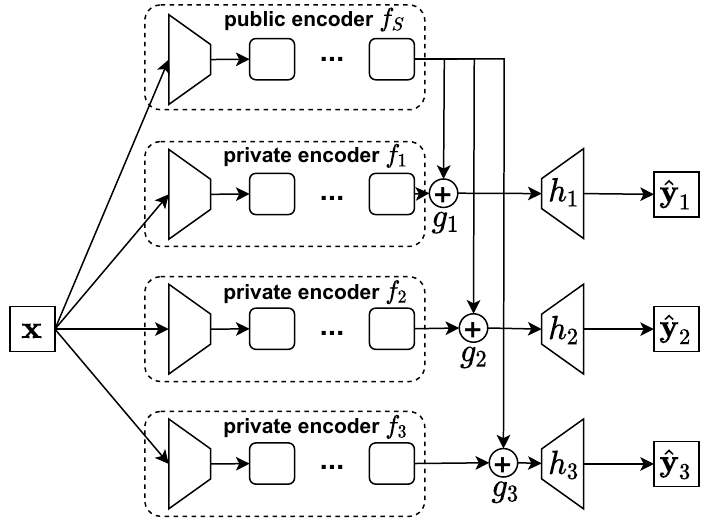}
\caption{An illustration for the architecture of the DSMTL model with three tasks (i.e., $m=3$). Here without loss of generality, we assume different tasks share the input data. For task $t$, an input $\mathbf{x}$  is first fed into both the public encoder $f_S$ and private encoder $f_t$, then it goes through the gate $g_t$ to obtain the combined feature representation, and finally it is through the private decoder $h_t$ to obtain the output $\hat{\mathbf{y}}_t$.}
\label{fig:smtl_model}
\end{figure}

\subsection{The Architecture}
\label{sec:DSMTL_arch}

As shown in Figure \ref{fig:smtl_model}, the architecture of the DSMTL model can be divided into four parts: a public encoder $f_S$ shared by all the tasks, $m$ private encoders $\{f_t\}_{t=1}^m$ for $m$ tasks, $m$ gates $\{g_t\}_{t=1}^m$ for $m$ tasks, and $m$ private decoders $\{h_t\}_{t=1}^m$ for $m$ tasks. For task $t$, its model consists of the public encoder $f_S$, the private encoder $f_t$, the gate $g_t$, and the private decoder $h_t$, where $f_S$ and $f_t$ are combined by $g_t$. Specifically, given a data sample $\mathbf{x}$, the gate $g_t$ in task $t$ receives two inputs: $f_S(\mathbf{x})$ and $f_t(\mathbf{x})$, and outputs $g_t(f_S(\mathbf{x}),f_t(\mathbf{x}))$, which is fed into $h_t$ to obtain the final prediction $h_t(g_t(f_S(\mathbf{x}),f_t(\mathbf{x})))$, which is used to define a loss for $\mathbf{x}$. Here the public encoder $f_S$ and private encoders $\{f_t\}_{t=1}^m$ usually have the same network structure. The private decoders are designed to be task-specific as different tasks may have different types of loss functions.

Here the gate $g_t$ is to determine the contributions of $f_S$ and $f_t$. Ideally, when task $t$ is unrelated to other tasks, $g_t$ should choose $f_t$ only. On another extreme where all the tasks have the same data distribution, all the tasks should use the same model and hence $g_t$ should choose $f_S$ only. In cases between those two extremes, $g_t$ can combine $f_S$ and $f_t$ in proportion. To achieve the aforementioned effects, we use a simple convex combination function for $g_t$ as
\begin{equation}
g_t(f_S(\mathbf{x}),f_t(\mathbf{x}))=\alpha_tf_S(\mathbf{x})+(1-\alpha_t)f_t(\mathbf{x}),\label{eq:gate_fun}
\end{equation}
where $\alpha_t\in \mathcal{M} = [0,1]$ defines the weight of $f_S(\mathbf{x})$ for task $t$ and it is a learnable parameter.
When $\alpha_t$ equals 0, only the private encoder $f_t$ will be used, which corresponds to the unrelated case. When $\alpha_t$ equals 1, only the public encoder $f_S$ will be used, which corresponds to the case that all the tasks follow identical or similar distributions. When $\alpha_t$ is between 0 and 1, $f_S$ and $f_t$ are combined with proportions $\alpha_t$ and $1-\alpha_t$, respectively, and $\alpha_t$ can be adaptively learned to minimize the training loss on task $t$.


The average empirical loss of all the tasks is defined as
\begin{align}
\frac{1}{mn} \sum_{t=1}^m\sum_{i=1}^n \mathcal{L}_t(\mathbf{y}^i_t,h_t( g_t(f_S(\mathbf{x}_t^i),f_t(\mathbf{x}_t^i)))),\label{DSMTL_obj_fun}
\end{align}
where $\mathbf{x}_t^i$ denotes the $i$th data point in task $t$, $\mathbf{y}^i_t$ denotes the label of $\mathbf{x}^i_t$ in task $t$, without loss of generality, different tasks are assumed to have the same number of data samples, which is denoted by $n$, and $\mathcal{L}_t$ denotes the loss function for task $t$ (e.g., the pixel-wise cross-entropy loss for the semantic segmentation task, the $L_1$ loss for the depth estimation task, and the element-wise dot product loss for the surface normal prediction task). The DSMTL model is to minimize the average empirical loss in Eq. \eqref{DSMTL_obj_fun} to learn its model parameters. In the following sections, we provide two learning strategies (e.g., individual learning and joint learning) to learn model parameters, leading to two variants of DSMTL, including DSMTL with Individual Learning (DSMTL-IL) and DSMTL with Joint Learning (DSMTL-JL).

\subsection{DSMTL-IL}

The set of all the parameters in the DSMTL model is denoted by $\Theta$. We divide $\Theta$ into $\Theta_S$ and $\Theta_H$, where  $\Theta_S$ includes all the parameters in $\{f_t\}_{t=1}^m$ and $\{h_t\}_{t=1}^m$, and $\Theta_H$ includes the parameters in $f_S$ and  $\{g_t\}_{t=1}^m$. The individual learning strategy consists of two stages. The first stage is to optimize $\Theta_H$ by fixing each $\alpha_t$ to 0. Then by fixing the learned $\Theta_H$ in the first stage, the second stage is to learn $\Theta_H$. As the single-task model for each task consists of an encoder and a decoder whose structures are identical to $f_t$ and $h_t$, respectively, the first stage is equivalent to learning a single-task model for each task, and after that, the second stage can learn the shared encoder $f_S$ and the gate $\{g_t\}_{t=1}^m$.
Formally, the objective function of the DSMTL-IL model is formulated as
\begin{equation}
\label{eqa: obj_al}
\min_{\Theta_H}\left[\min_{\Theta_S}\frac{1}{mn}\sum_{t=1}^m\sum_{i=1}^n  \mathcal{L}_t(\mathbf{y}^i_t,h_t( g_t(f_S(\mathbf{x}_t^i),f_t(\mathbf{x}_t^i))))\right].
\end{equation}


The DSMTL-IL model can be proved to achieve both \emph{empirically individual safe multi-task learning} and \emph{probably individual safe multi-task learning} as shown in Section \ref{sec:analysis_DSMTL_IL} due to its two-stage optimization process.  


\subsection{DSMTL-JL}

Different from the DSMTL-IL model which optimizes two partitions of model parameters sequentially, the DSMTL-JL model adopts a joint learning strategy to learn $\Theta$ together.
Formally, the objective function of the DSMTL-JL model is formulated as
\begin{equation} \label{eqa: obj}
\min_{\Theta} \frac{1}{mn}\sum_{t=1}^m\sum_{i=1}^n  \mathcal{L}_t(\mathbf{y}^i_t,h_t( g_t(f_S(\mathbf{x}^i_t),f_t(\mathbf{x}^i_t)))).
\end{equation}

The joint learning strategy allows the DSMTL model to learn all the model parameters,
which is more flexible for deep neural networks in an end-to-end learning manner. Different from the DSMTL-IL model, the DSMTL-JL model can achieve both \emph{empirically average safe multi-task learning} and \emph{probably average safe multi-task
learning} as proven in Section \ref{sec:analysis_DSMTL_JL}.

\section{Analyses}\label{sec:analysis}

In this section, we provide theoretical analyses to analyze the safeness of both the DSMTL-IL and DSMTL-JL methods. 

\subsection{Preliminary}

With $m$ tasks, the probability measure for the data distribution in task $t$ is denoted by $\mu_t$ and the data in all the tasks take the form of $(\bar{\mathbf{X}},\bar{\mathbf{Y}}) \sim \prod_{t=1}^m(\mu_t)^n$, where $\mathbf{X}_t = (\mathbf{x}_{t}^{1},\ldots,\mathbf{x}_{t}^{n})$ denotes the data in task $t$, $\bar{\mathbf{X}} =(\mathbf{X}_1,\ldots,\mathbf{X}_m)$,  and $\bar{\mathbf{Y}}$ denotes labels for $\bar{\mathbf{X}}$.
Here we consider the encoders $f_1,\ldots,f_m,f_S: \mathcal{X}\to \mathbb{R}^{q}$ as mapping functions chosen from a hypothesis class $\mathcal{F}$ and the decoders $h_1,\ldots,h_m$
as mapping functions chosen from a hypothesis class $\mathcal{H}$.
To facilitate the analysis, we introduce following assumption.
\begin{assumption}\label{assumption:1}
Assume that $(i)$ $\mathcal{L}_t(\cdot,\cdot)\in [0,1]$ for $t=1,\ldots,m$ is $1$-Lipschitz w.r.t the second argument; $(ii)$ The hypothesis class $\mathcal{F}$ is uniformly bounded; $(iii)$ The functions in hypothesis class $\mathcal{H}$ are Lipschitz continuous; $(iv)$ $0\in \mathcal{F}$ and $h(0)= 0$ holds for all the functions $h$ in $\mathcal{H}$.\footnote{The last assumption in Assumption \ref{assumption:1} is not essential but it can help give simpler theoretical results as verified by the proofs in the appendix.}
\end{assumption}


To analyze the safeness of the DSMTL variants, we compare with the corresponding Single-Task Learning (STL) model, which consists of an encoder and a decoder with identical structures to $f_t$ and $h_t$, respectively, for task $t$. We also compare with the widely-used Hard Parameter Sharing (HPS) model, which consists of a shared encoder and task-specific decoders with the network structures identical to $f_S$ and $\{h_t\}_{t=1}^m$, respectively.
Then the empirical loss of the STL model on task $t$ is formulated as $L^{\text{STL}}_t=\frac{1}{n} \sum_{i=1}^n \mathcal{L}_t(\mathbf{y}^i_t,h^{\text{STL}}_t(f_t^{\text{STL}}(\mathbf{x}_t^i)))$ and the expected loss of the STL model on task $t$ is formulated as $\mathcal{E}^{\text{STL}}_t=\mathbb{E}_{(\mathbf{x},\mathbf{y})\sim \mu_t}[\mathcal{L}_t(\mathbf{y}, h^{\text{STL}}_t(f_t^{\text{STL}}(\mathbf{x})))]$, where $f_t^{\text{STL}}$ and $h^{\text{STL}}_t$ have identical network structures to $f_t$ and $h_t$ in DSMTL, respectively.
The average empirical loss of the STL model is computed as
$L^{\text{STL}} = \frac{1}{m} \sum_{t=1}^mL^{\text{STL}}_t,$
and the average expected loss of the STL model is as
$\mathcal{E}^{\text{STL}} = \frac{1}{m}\sum_{t=1}^m\mathcal{E}_t^{\text{STL}}
$.
Similarly, the average empirical loss of the HPS model is formulated as
\begin{equation*}
L^{\text{HPS}} = \frac{1}{mn}\sum_{t=1}^m\sum_{i=1}^n  \mathcal{L}_t(\mathbf{y}^i_t,h^{\text{HPS}}_t(f^{\text{HPS}}_S(\mathbf{x}_t^i))),
\end{equation*}
where $h^{\text{HPS}}_t$ and $f^{\text{HPS}}_S$ have identical network structures to $h_t$ and $f_S$ in DSMTL, respectively.

The expected loss of DSMTL on task $t$ is formulated  as
\begin{equation*}
\mathcal{E}_t = \mathbb{E}_{(\mathbf{x},\mathbf{y})\sim \mu_t}[\mathcal{L}_t(\mathbf{y}, h_t(g_t(f_S(\mathbf{x}), f_t(\mathbf{x}))].
\end{equation*}
Then the average expected loss of the DSMTL model is computed as $\mathcal{E}=\frac{1}{m}\sum_{t=1}^m\mathcal{E}_t$.

\subsection{Analysis on DSMTL-IL}
\label{sec:analysis_DSMTL_IL}

In DSMTL-IL, since $\alpha_t$ is set to zero for each task in the first stage, the objective function of the first stage is equivalent to the following problem as
\begin{equation}\label{eq:stl_smtl}
\min_{\Theta_S}\ \frac{1}{mn}\sum_{t=1}^m\sum_{i=1}^n  \mathcal{L}_t(\mathbf{y}^i_t,h_t( g_t^0(\emptyset,f_t(\mathbf{x}_t^i)))),
\end{equation}
where $g_t^{0}$ denotes the gate of task $t$ with $\alpha_t$ as 0 and $\emptyset$ denotes a null network. Thus the first stage is to train the STL model with the empirical loss $L^{\text{STL}}_t$ for task $t$. As $g_t$ is a learnable gate, after sufficient training in the second stage, the empirical loss of the DSMTL-IL model on each task is no larger than that of the first stage. Based on this observation, we have the following theorem.\footnote{The proofs for all the theorems are put in the appendix.}

\begin{theorem}\label{thm1}
Let $L^*$ be the optimal value of  problem \eqref{eqa: obj_al} and $L_t^*$ the corresponding empirical loss for task $t$. Then we have $L_t^*\le L^{\text{STL}}_t$ for all $1\le t \le m$.
\end{theorem}
Theorem \ref{thm1} shows that the DSMTL-IL model can achieve \emph{empirically individual safe multi-task learning} in Definition \ref{definition_empirically_safe_multi-task_learning}. Moreover, in the following theorem, we show that it also achieves \emph{probably individual safe multi-task learning}.
\begin{theorem}\label{thm2}
Suppose Assumption \ref{assumption:1} is satisfied. Let $L^*$ be the optimal value of problem \eqref{eqa: obj_al} and the corresponding solution is denoted by $\hat{f}_S$, $\{\hat{f}_t\}$, $\{\hat{h}_t\}$, $\{\hat{g}_t\}$. Let $\hat{\mathcal{E}}_t = \mathcal{E}_t(\hat{f}_S,\hat{f}_t, \hat{h}_t, \hat{g}_t )$. Then for $(\mathbf{X}_t,\mathbf{Y}_t) \sim \mu_t^n$, with probability at least $1-\delta$, we have $$\hat{\mathcal{E}}_t+\epsilon_t \le \mathcal{E}^{\text{STL}}_t + \rho_{n,t},$$ where $\epsilon_t$ is formulated as $\epsilon_t = L^{\text{STL}}_t - L_t^*$, $\rho_{n,t}$ is defined as
$\rho_{n,t} = \frac{C_1 G(\mathcal{F}'(\mathbf{X}_t))}{n} +\frac{ C_2Q}{\sqrt{n}}+\sqrt{\frac{18\ln{\frac{2}{\delta}}}{n}}$,
$\mathcal{F}'(\mathbf{X}_t)=\{ (f(\mathbf{x}_t^i)):f \in \mathcal{F}\}$, $C_1$ and $C_2$ are two constants, $G(\cdot)$ denotes the Gaussian average \cite{bartlett2002rademacher}, and $Q$ is defined as
$Q = \sup_{z\not = \tilde{z} \in \mathbb{R}^{nq}} \mathbb{E} \sup_{h\in \mathcal{H}}
\frac{\langle\gamma,h(z_{i}) - h(\tilde{z}_{i}) \rangle}{\|z-\tilde{z}\|}$
with $\gamma$ as a vector of independent standard normal variables.
\end{theorem}

\begin{remark}

For many classes of interest, the Gaussian average $G(\mathcal{F}'(\mathbf{X}_t))$ is $O(\sqrt{n})$ according to \cite{maurer2016chain}. For reasonable classes $\mathcal{H}$, one can find a bound on $Q$, which is independent of $n$ \cite{maurer2016benefit}. Therefore, $\rho_{n,t}$ is $O(1/\sqrt{n})$ for $1\le t \le m$ and hence $\rho_{n,t}$ satisfies $\lim_{n\to +\infty}\rho_{n,t} = 0$. Moreover, according to Theorem \ref{thm1}, we have $\epsilon_t\ge 0$. Thus, Theorem \ref{thm2} proves that the proposed DSMTL-IL model can achieve \emph{probably individual safe multi-task learning} in Definition \ref{definition_expected_safe_multi-task_learning}.

\end{remark}

\subsection{Analysis on DSMTL-JL}
\label{sec:analysis_DSMTL_JL}

In this section, we analyze the safeness and excess risk bound of the DSMTL-JL model.

For the safeness of the DSMTL-JL model, we have the following theorems.

\begin{theorem}\label{thm3}
Let $L^*$ be the optimal value of problem \eqref{eqa: obj}. Then we have $L^*$ is no higher than the minimun of $L^{\text{STL}}$ and $L^{\text{HPS}}$, i.e.,  $L^* \le \min\{L^{\text{STL}},L^{\text{HPS}}\}$.
\end{theorem}

\begin{remark}

Theorem \ref{thm3} shows that the DSMTL-JL model can achieve \emph{empirically average safe multi-task learning} in Definition \ref{definition_empirically_safe_multi-task_learning}. Moreover, it also implies that it can achieve a lower average empirical loss compared with the corresponding HPS model. To see that, the HPS model for task $t$ can be represented as $h_t(g_t^{1}(f_S(\mathbf{x}),\emptyset))$, where $g_t^{1}$ denotes the gate of task $t$ with $\alpha_t$ as 1. As $g_t^{1}$ is a feasible solution for the DSMTL-JL model, it is easy to see that the empirical loss of the DSMTL-JL model after sufficient training is lower than that of the HPS model, which could be one reason why the DSMTL-JL model outperforms the HPS model as shown in experiments.

\end{remark}


\begin{theorem}\label{thm4}
Suppose Assumption \ref{assumption:1} is satisfied. Let $L^*$ be the optimal value of problem \eqref{eqa: obj} and the corresponding solution is denoted by $\hat{f}_S$, $\{\hat{f}_t\}$, $\{\hat{h}_t\}$, $\{\hat{g}_t\}$. Let $\hat{\mathcal{E}} = \mathcal{E}(\hat{f}_S,\{\hat{f}_t\}, \{\hat{h}_t\}, \{\hat{g}_t\})$. Then for $(\bar{\mathbf{X}},\bar{\mathbf{Y}}) \sim \prod_{t=1}^m(\mu_t)^n$, with probability at least $1-\delta$, we have $$\hat{\mathcal{E}}+\epsilon \le \mathcal{E}^{\text{STL}} + \rho_n,$$
where
$\rho_n =  \frac{C_1 G(\mathcal{F}(\bar{\mathbf{X}}))}{nm}
+ \frac{C_2Q}{\sqrt{n}} +\sqrt{\frac{18\ln{\frac{2}{\delta}}}{mn}}$, $\mathcal{F}(\bar{\mathbf{X}}) = \{ (f_1(\mathbf{x}_t^i),\ldots,f_m(\mathbf{x}_t^i)):f_t \in \mathcal{F}\}$, $C_1$ and $C_2$ are two constants, $\epsilon$ is formulated as $\epsilon = L^{\text{STL}} - L^*$, and $Q$ is defined in Theorem \ref{thm2}.
\end{theorem}

\begin{remark}

According to \cite{maurer2016chain}, the Gaussian average $G(\mathcal{F}(\bar{\mathbf{X}}))$ is $O(\sqrt{mn})$ for many classes of interest. Therefore, $\rho_n$ is $O(\frac{1}{\sqrt{n}})$ and it satisfies $\lim_{n\to +\infty}\rho_n = 0$. According to Theorem \ref{thm3}, $\epsilon\ge 0$. Thus, Theorem \ref{thm4} implies that the proposed DSMTL-JL model can achieve \emph{probably average safe multi-task learning} in Definition \ref{definition_expected_safe_multi-task_learning}.

\end{remark}

To analyze the excess risk bound for the DSMTL-JL model, we define the minimal expected risk as
\begin{equation*}
\mathcal{E}^* = \min_{f_S,f_t\in \mathcal{F}, h_t\in \mathcal{H},\alpha_t \in \mathcal{M}} \mathcal{E}.
\end{equation*}
Then we have the following result.

\begin{theorem}
\label{thm5}
Suppose Assumption \ref{assumption:1} is satisfied. The solution of the optimization problem in Eq. \eqref{eqa: obj} is denoted by $\hat{f}_S$, $\{\hat{f}_t\}$, $\{\hat{h}_t\}$, $\{\hat{g}_t\}$. Let $\hat{\mathcal{E}} = \mathcal{E}(\hat{f}_S,\{\hat{f}_t\}, \{\hat{h}_t\}, \{\hat{g}_t\})$. Then for $(\bar{\mathbf{X}},\bar{\mathbf{Y}}) \sim \prod_{t=1}^m(\mu_t)^n$, with probability at least $1-\delta$, we have
\begin{equation}\label{thm5:eq}
\hat{\mathcal{E}}- \mathcal{E}^* \le \frac{C_1 G(\mathcal{F}(\bar{\mathbf{X}}))}{nm}
+ \frac{C_2Q}{\sqrt{n}} +\sqrt{\frac{8\ln\frac{4}{\delta}}{mn}},
\end{equation}
where $\mathcal{F}(\bar{\mathbf{X}})$ is defined in Theorem \ref{thm4}, $C_1$ and $C_2$ are two constants, and $Q$ is defined in Theorem \ref{thm2}.
\end{theorem}

Theorem \ref{thm5} provides an upper bound on the error of this DSMTL-JL model. In the bound (\ref{thm5:eq}), the first term of the right-hand side can be regarded as the cost to estimate all the feature mappings $\{f_t\}$ and $f_S$, and it decreases with respect to the number of tasks. The order of this term is $O(\frac{1}{\sqrt{mn}})$. The second term of the right-hand side corresponds to the cost to estimate task-specific functions $\{g_t\}$ and $\{h_t\}$, and it is of order $O(\frac{1}{\sqrt{n}})$. The third term defines the confidence of the bound. The convergence rate of this bound is as tight as typical generalization bounds \cite{maurer2016benefit} for MTL.

\section{Architecture Learning for DSMTL}
\label{sec:DSMTL_AL}

A limitation of the DSMTL model is that its model size grows linearly with respect to the number of tasks, which makes its scalability not so good. To address this issue, we propose the DSMTL-AL model to not only achieve comparable or even better performance than the DSMTL-IL and DSMTL-JL models but also learn a more compact architecture via techniques in neural architecture search \cite{DARTS}.


Inspired by the architecture of the DSMTL model introduced in Section \ref{sec:DSMTL_arch}, the supernet in the DSMTL-AL method has a public encoder $f_S$, $m$ private encoders $\{f_t\}_{t=1}^m$, and $m$ private decoders $\{h_t\}_{t=1}^m$.
Instead of treating the public and private encoders as a whole, we divide them into modules, which could be a fully connected layer or a sophisticated ResNet block/layer, depending on the MTL problem under investigation.
Without loss of generality, we assume that both the public and private encoders have the same number of modules, i.e., $(P-1)$.

\begin{figure}[h]
\centering
\includegraphics[width=\linewidth]{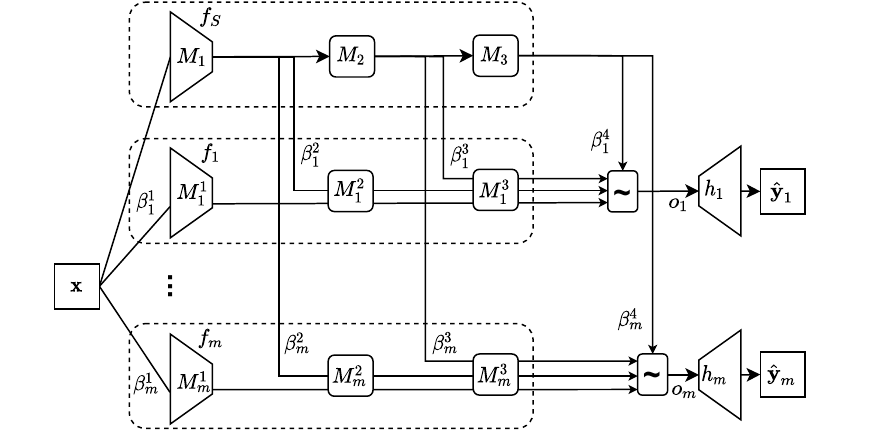}
\caption{An illustration for the architecture learning process in the DSMTL-AL model, where without loss of generality, different tasks are assumed to share the same input data. For simplicity, we assume that there are three modules (i.e., $P=4$) in each encoder.
``$\sim$" represents the convex combination.
In the retraining process, only the branch with the largest architecture parameter whose index is denoted by $p_t^*$ is preserved and the first $(p_t^*-1)$ private modules will be removed.}
\label{fig:dsmtl_al}
\end{figure}

As illustrated in Figure \ref{fig:dsmtl_al}, the DSMTL-AL method is to learn a branching architecture to combine public modules in the public encoder and private modules in the corresponding private encoder for each task to reduce the model size. The search space for the architecture in the DSMTL-AL method consists of $m$ architecture parameters $\{\bm{\beta}_t\}^m_{t=1}$ to decide branch positions for $m$ tasks, where $\bm{\beta}_t=(\beta_t^1,\ldots,\beta_t^P)$.
As a binary parameter, $\beta_t^p\in\{0,1\}$ indicates whether task $t$ branches at the branch position $p$ from $f_S$ to $f_t$. Specifically, when $\beta_t^p$ equals 1, there will be a branch to feed the output of the $(p-1)$-th module in $f_S$ to the $p$-th module in $f_t$ to form a combined encoder for task $t$. Moreover, the sum of entries in $\bm{\beta}_t$ should be 1 for any $t$, indicating that there is only one branch position for each task. In this sense, only part of the private/public modules in all the encoders will be used for each task. Hence the model size is smaller than the entire supernet, which is used in the DSMTL-IL and DSMTL-JL models.

The search space in the DSMTL-AL method includes both STL and HPS architectures as two extremes. When all the tasks are unrelated to each other, the architecture in the DSMTL-AL method could become the STL architecture by choosing $\{f_t\}$ for each task (i.e., $\beta_t^1=1$ for $t=1,\ldots,m$). For highly related or even identical tasks, the architecture of the DSMTL-AL method could become the HPS architecture by choosing $f_S$ only (i.e., $\beta_t^P=1$ for $t=1,\ldots,m$). 


The DSMTL-AL method is to find the best branching architecture in the search space for all the $m$ tasks by learning architecture parameters $\{\bm{\beta}_t\}_{t=1}^m$.
If task $t$ branches at the branch position $p$ of $f_S$ to connect to the corresponding next module in $f_t$, $\beta_t^p$ is set to 1 and all the $\beta_t^i$'s ($i\ne p$) are set to 0. In this case, the output of the combined encoder for task $t$ consisting of the first $(p-1)$ public modules in $f_S$ and the last $(P-p)$ private modules in $f_t$ is $f_t(f_S(\mathbf{x},p-1),p,P-1)$, where $f_S(\cdot,p)$ denotes the output of the $p$-th module in $f_S$ and $f_t(\cdot,p,q)$ denotes the output of the $q$-th module in $f_t$ starting from the $p$-th module. Here $f_S(\mathbf{x},0)$ is defined to be $\mathbf{x}$, corresponding to the input to the first module, and $f_t(\mathbf{x},P,P-1)$ is defined as $\mathbf{x}$.
Since $\{\bm{\beta}_t\}$ are binary variables, this discrete nature makes stochastic gradient descent methods incapable of learning them.
Here we relax $\{\bm{\beta}_t\}$ to be continuous and define them as the probability of branching at each branch position. Specifically, the output of the combined encoder for task $t$ is formulated as
\begin{equation*}
\mathbf{o}_t(\mathbf{x},\bm{\beta}_t) = \sum^{P}_{p=1}
\beta^p_t f_t(f_S(\mathbf{x},p-1),p,P-1),
\end{equation*}
where $\bm{\beta}_t$ is in the $(P-1)$-dimensional simplex set denoted by $\mathcal{S}_P$, satisfying that $\beta^p_t\ge 0$ and $\sum_{p=1}^P\beta^p_t=1$.
Here $\mathbf{o}_t(\mathbf{x},\bm{\beta}_t)$ is a convex combination of outputs of all possible branching architectures weighted by probabilities based on $\bm{\beta}_t$.
Then $\mathbf{o}_t(\mathbf{x},\bm{\alpha}_t)$ is fed into the decoder $h_t$ to generate the prediction and hence the empirical loss for task $t$ is formulated as
\begin{equation*}
L_t({\Theta}_t,\bm{\beta}_t) = \frac{1}{n}\sum_{i=1}^n\mathcal{L}_t(\mathbf{y}^i_t,h_t(\mathbf{o}_t(\mathbf{x}_t^i,\bm{\beta}_t)))
\end{equation*}
where $\Theta_t$ includes all the parameters in $f_S$, $f_t$, and $h_t$.
Then the weighted empirical loss over $m$ tasks is formulated as
\begin{equation}
L(\Theta,\bm{\beta},\bm{w}) = \sum^m_{t=1} w_t L_t (\Theta_t,\bm{\beta}_t),\label{eq:smtlba_lw}
\end{equation}
where $\Theta=\{\Theta_t\}_{t=1}^m$, $\bm{w}=(w_1,\ldots,w_m)$ is in $\mathcal{S}_m$ satisfying $w_t\ge 0$ and $\sum_{t=1}^m w_t=1$, and $\bm{\beta}=(\bm{\beta}_1,\ldots,\bm{\beta}_m)$. $\bm{w}$ specifies the weighting among all the tasks. Setting them to $\frac{1}{m}$ as in the DSMTL-IL and DSMTL-JL models may lead to suboptimal performance and hence we aim to learn them directly.

Here $\Theta$ is viewed as model parameters, while $\bm{\beta}$ and $\bm{w}$ are hyperparameters. To learn all of them, we adopt a bi-level formulation as
\begin{align}
\min_{\bm{\beta}\in\mathcal{S}_P, \bm{w}\in\mathcal{S}_m} & \   L_{val}(\hat{\Theta},\bm{\beta}, \bm{w})\nonumber\\
\mathrm{s.t.} & \  \hat{\Theta} = {\arg\min}_{\Theta}~L_{tr}(\Theta,\bm{\beta}, \bm{w}),
\label{eq:darts_loss}
\end{align}
where the entire training dataset is divided into a training set and a validation set, $\mathcal{L}_{tr}(\cdot,\cdot,\cdot)$ denotes the weighted empirical loss defined in Eq. (\ref{eq:smtlba_lw}) on the training set, and $\mathcal{L}_{val}(\cdot,\cdot,\cdot)$ denotes the weighted empirical loss defined in Eq. (\ref{eq:smtlba_lw}) on the validation set. Here the constraints on $\bm{\beta}$ and $\mathbf{w}$ can be alleviated via the reparameterization based on the softmax function. We adopt the gradient-based hyperparameter optimization algorithm in \cite{franceschi2018bilevel,DARTS}
to solve problem (\ref{eq:darts_loss}) with the first-order approximation.
After solving the problem (\ref{eq:darts_loss}), we can learn  architecture for task $t$ by determining the branch position as
\begin{equation*}
\beta^*_t = {\arg\max}_p(\{\beta_t^p\}^P_{p=1}).
\end{equation*}

With the learned architecture, we can use the entire training dataset to retrain the model parameters $\Theta$. Inspired by the gating mechanism in the DSMTL-IL and DSMTL-JL models, the final encoder for task $t$ consists of the shared encoder and the combined encoder determined by $\beta_t^*$. One reason for that is that the shared encoder $f_S$ could be fully used to improve the performance with little increase or even no increase in the model size, which is due to that all the modules in $f_S$ will usually be chosen by at least one task during the architecture learning process. Formally, the final encoder for task $t$ is formulated as
\begin{align}
\hat{g}_t(\mathbf{x},\beta^*_t) &= \alpha_t f_S(\mathbf{x}) + (1-\alpha_t) f_t(f_S(\mathbf{x}, \beta^*_t-1), \beta^*_t,P-1),\label{DSMTL_AL_final_encoder}
\end{align}
where
with abuse of notations, $\alpha_t\in\mathcal{M}=[0,1]$ is a learnable parameter to measure the weight of $f_S(\mathbf{x})$ and acts similarly to $\alpha_t$ in the DSMTL-IL and DSMTL-JL models.
Mathematically, the objective function of the retraining process is formulated as
\begin{equation}
\min_{\Theta,\bm{\alpha}\in\mathcal{M}} \sum^m_{t=1} \frac{w_t}{n}\sum_{i=1}^n\mathcal{L}_t(\mathbf{y}^i_t,h_t(\hat{g}_t(\mathbf{x}^i_t,\beta^*_t))),
\end{equation}
where $\Theta$ denotes all the model parameters, $\bm{\alpha}=(\alpha,\ldots,\alpha_m)$, and $w_t$ is the loss weight learned from problem (\ref{eq:darts_loss}) for task $t$. After the retraining process, we can use the learned model parameters to make prediction for each task.

Though the DSMTL-AL model cannot be theoretically proved to achieve some version of \emph{safe multi-task learning}, as shown in the next section, empirically it not only achieves good performance but also learns compact architectures.

\section{Experiments}

In this section, we evaluate the proposed models.

\subsection{Datasets and Evaluation Metrics} \label{app:dataset}

Experiments are conducted on four MTL CV datasets, including 
CityScapes \cite{cordts2016cityscapes}, NYUv2 \cite{silberman2012indoor}, PASCAL-Context \cite{MottaghiCLCLFUY14}, and Taskonomy \cite{zamir2018taskonomy}. 

The CityScapes dataset consists of high resolution outside street-view images. By following \cite{liu2019end}, we evaluate the performance on the $7$-class semantic segmentation and depth estimation tasks. The NYUv2 dataset consists of \mbox{RGB-D} indoor scene images from three learning tasks: $13$-class  semantic segmentation, depth estimation, and surface normal prediction. The PASCAL-Context dataset is an annotation extension of the PASCAL VOC 2010 challenge with four learning tasks: $21$-class semantic segmentation, $7$-class human parts segmentation, saliency estimation, and surface normal estimation, where the last two tasks are generated by \cite{ManinisRK19}. The Taskonomy dataset contains indoor images. By following \cite{standley2020tasks}, we sample five learning tasks, including $17$-class semantic segmentation, depth estimation, keypoint detection, edge detection, and surface normal prediction. 

The semantic segmentation task on the PASCAL-Context dataset is evaluated by the mean Intersection over Union (mIoU) by following \cite{ManinisRK19}. On the other three CV datasets, this task is additionally evaluated in terms of the Pixel Error (abbreviated as `Pix Err') by following \cite{SunPFS20}. For the depth estimation task, the absolute error (abbreviated as `Abs Err') and relative error (abbreviated as `Rel Err') are used as the evaluation metrics. For the surface normal prediction task, the mean and median angle distances between the prediction and ground truth of all pixels are used as measures. For this task, the percentage of pixels, whose prediction is within the angles of $11.25^{\circ}$, $22.5^{\circ}$, and $30^{\circ}$ to the ground truth, is used as another measure. For the keypoint detection and edge detection tasks, the absolute error (abbreviated as `Abs Err') is used as the evaluation metric. For the human parts segmentation task, the mIoU is used as the measure. For the saliency estimation task, the mIoU and max F-measure (maxF) are adopted as the evaluation metrics.

As introduced above, for each task, we use one or more evaluation metrics to thoroughly evaluate the performance.
To better show the comparison between each method and STL, we compute the relative performance of each method over STL in terms of the $j$th evaluation metric on task $t$ as $\Delta_{t,j}=(-1)^{p_{t,j}}(\mathrm{M}_{t,j}-\mathrm{STL}_{t,j})$, where for a method $\mathrm{M}$, $\mathrm{M}_{t,j}$ denotes its performance in terms of the $j$th evaluation metric for task $t$, $\mathrm{STL}_{t,j}$ is defined similarly, $p_{t,j}$ equals 1 if a lower value represents better performance in terms of the $j$th metric in task $t$ and $0$ otherwise. So positive relative performance indicates better performance than STL.
The overall relative improvement of a method $\mathrm{M}$ over STL is defined as $\Delta_I = \frac{1}{m}\sum_{t=1}^m \frac{1}{m_t}\sum_{j=1}^{m_t} \frac{\Delta_{t,j}} {\mathrm{STL}_{t, j}}$, where $m_t$ denotes the number of evaluation metrics in task $t$.


To empirically measure the safeness of each model, we define the safeness coefficient $\eta$ for a model as the proportion of tasks on which this model empirically performs no worse than the STL model.
Formally, $\eta$ is formulated as ${\eta} = \frac{1}{m}\sum_{t=1}^m \frac{1}{m_t}\sum_{j=1}^{m_t} \delta(\Delta_{t,j})\times 100$, where $\delta(x)$ is the delta function that outputs 0 when $x < 0$ and otherwise 1.
Obviously, $\eta$, whose maximum is 100, is expected to be as large as possible.

\begin{table}[!htb]
\centering
\caption{Performance of various models on the CityScapes validation dataset. $\uparrow$ ($\downarrow$) indicates the higher (lower) the result, the better the performance. The green color indicates that the corresponding method performs better than the STL method and the red color indicates oppositely. The number of parameters (abbreviated as Parms.) is calculated in MB.}
\vskip -0.1in
\resizebox{\linewidth}{!}{
\begin{tabular}{lccccccc}
\toprule
 \multirow{2}{*}{Method} & \multicolumn{2}{c}{Segmentation} & \multicolumn{2}{c}{Depth} & \multirow{2}{*}{$\Delta_I \uparrow$} & \multirow{2}{*}{${\eta}\uparrow$} & \multirow{2}{*}{Parms. (M)$\downarrow$} \\
 \cmidrule(r){2-3} \cmidrule(r){4-5}
 & {mIoU $\uparrow$} &  {Pix Err $\downarrow$} &  {Abs Err $\downarrow$} &  {Rel Err$\downarrow$} \\
\cmidrule(lr){1-8}
STL  & $67.48$ & $9.00$ & $0.0139$  & $46.2507$ & 0 & - &79.27\\
\cmidrule(lr){1-8}
HPS &
\colorbox{neg}{$-0.08$} &\colorbox{neg}{$-0.08$} &\colorbox{neg}{$-0.0003$} &\colorbox{pos}{$+0.8245$} &$-0.0035$ &25 &55.76 \\
Cross-stitch &
\colorbox{pos}{$+0.53$} &\colorbox{pos}{$+0.29$} &\colorbox{pos}{$+0.0004$} &\colorbox{pos}{$+1.8261$} &$+0.0271$ &100 &79.26 \\
MTAN &
\colorbox{pos}{$+1.49$} &\colorbox{pos}{$+0.59$} &\colorbox{pos}{$+0.0003$} &\colorbox{pos}{$+2.4999$} &$+0.0408$ &100 &72.04 \\
NDDR-CNN &
\colorbox{pos}{$+0.54$} &\colorbox{pos}{$+0.25$} &\colorbox{pos}{$+0.0002$} &\colorbox{pos}{$+1.3845$} &$+0.0200$ &100 &101.58 \\
AFA &
\colorbox{pos}{$+1.44$} &\colorbox{pos}{$+0.52$} &\colorbox{neg}{$-0.0019$} &\colorbox{neg}{$-0.9136$} &$-0.0193$ &50 &87.09 \\
RotoGrad &
\colorbox{pos}{$+0.69$} &\colorbox{pos}{$+0.37$} &\colorbox{pos}{$+0.0004$} &\colorbox{pos}{$+1.2729$} &$+0.0274$ &100 &57.86 \\
\cmidrule(lr){1-8}
MaxRoam &
\colorbox{neg}{$-0.76$} &\colorbox{pos}{$+0.33$} &\colorbox{neg}{$-0.0003$} &\colorbox{pos}{$+3.9607$} &$+0.0224$ &50 &55.76 \\
MTL-NAS &
\colorbox{neg}{$-2.13$} &\colorbox{neg}{$-0.80$} &\colorbox{neg}{$-0.0005$} &\colorbox{neg}{$-1.6786$} &$-0.0482$ &0 &87.26 \\
BMTAS &
\colorbox{pos}{$+2.09$} &\colorbox{pos}{$+0.81$} &\colorbox{pos}{$+0.0017$} &\colorbox{pos}{$+1.9947$} &$+0.0723$ &100 &79.04 \\
TSN &
\colorbox{pos}{$+2.27$} &\colorbox{pos}{$+0.85$} &\colorbox{pos}{$+0.0011$} &\colorbox{pos}{$+0.4047$} &$+0.0544$ &100 &50.58 \\
LTB &
\colorbox{pos}{$+2.24$} &\colorbox{pos}{$+0.82$} &\colorbox{pos}{$+0.0014$} &\colorbox{neg}{$-0.4489$} &$+0.0539$ &75 &70.73 \\
\cmidrule(lr){1-8}
DSMTL-IL &
\colorbox{pos}{$+2.94$} &\colorbox{pos}{$+1.13$} &\colorbox{pos}{$+0.0015$} &\colorbox{pos}{$+0.4879$} &$+0.0715$ &100 &102.78 \\
DSMTL-JL &
\colorbox{pos}{$+1.50$} &\colorbox{pos}{$+0.62$} &\colorbox{pos}{$+0.0005$} &\colorbox{pos}{$+3.3886$} &$+0.0501$ &100 &102.78 \\
DSMTL-AL &
\colorbox{pos}{$+3.07$} &\colorbox{pos}{$+1.18$} &\colorbox{pos}{$+0.0017$} &\colorbox{pos}{$+4.3300$} &$+0.0988$ &100 &79.04 \\
\bottomrule
\end{tabular}
}
\label{tab:cityscapes}
\end{table}

\begin{table*}[!htbp]
\centering
\caption{Performance of various models on the NYUv2 validation dataset. $\uparrow$ ($\downarrow$) indicates the higher (lower) the result, the better the performance. The green color indicates that the corresponding method performs better than the STL method and the red color indicates oppositely. The number of parameters (abbreviated as Parms.) is calculated in MB.}
\vskip -0.1in
\resizebox{\linewidth}{!}{
\begin{tabular}{lcccccccccccc}
\toprule
\multirow{3}{*}{Method} & \multicolumn{2}{c}{Segmentation} & \multicolumn{2}{c}{Depth} & \multicolumn{5}{c}{Surface Normal} & \multirow{3}{*}{$\Delta_I\uparrow$} & \multirow{3}{*}{${\eta}$} & \multirow{3}{*}{Parms. (M)$\downarrow$} \\
\cmidrule(r){2-3} \cmidrule(r){4-5} \cmidrule(r){6-10}
& \multirow{2}{*}{mIoU $\uparrow$} &  \multirow{2}{*}{Pix Err $\downarrow$} &  \multirow{2}{*}{Abs Err $\downarrow$} &  \multirow{2}{*}{Rel Err$\downarrow$} & \multicolumn{2}{c}{Angle Distance $\downarrow$} & \multicolumn{3}{c}{Within $t^{\circ}$ $\uparrow$} \\
\cmidrule(r){6-7} \cmidrule(r){8-10} & & & & & Mean  & Median   & 11.25  & 22.5  & 30   \\
\cmidrule(lr){1-13}
STL  & $53.11$ & $4.80$ & $0.3957$ & $0.1632$ & $22.26$ & $15.49$ & $38.61$ & $64.43$ & $74.69$ & 0 & - & 118.91\\
\cmidrule(lr){1-13}
HPS &
\colorbox{pos}{$+1.37$} &\colorbox{pos}{$+0.62$} &\colorbox{pos}{$+0.0118$} &\colorbox{pos}{$+0.0084$} &\colorbox{neg}{$-1.24$} &\colorbox{neg}{$-1.57$} &\colorbox{neg}{$-3.30$} &\colorbox{neg}{$-3.33$} &\colorbox{neg}{$-2.55$} &$+0.0001$ &67 &71.89 \\
Cross-stitch &
\colorbox{pos}{$+0.35$} &\colorbox{pos}{$+0.29$} &\colorbox{pos}{$+0.0153$} &\colorbox{pos}{$+0.0077$} &\colorbox{neg}{$-0.75$} &\colorbox{neg}{$-0.84$} &\colorbox{neg}{$-1.60$} &\colorbox{neg}{$-2.01$} &\colorbox{neg}{$-1.68$} &$+0.0051$ &67 &118.89 \\
MTAN &
\colorbox{pos}{$+1.63$} &\colorbox{pos}{$+0.58$} &\colorbox{pos}{$+0.0161$} &\colorbox{pos}{$+0.0083$} &\colorbox{neg}{$-0.71$} &\colorbox{neg}{$-0.81$} &\colorbox{neg}{$-1.70$} &\colorbox{neg}{$-1.80$} &\colorbox{neg}{$-1.38$} &$+0.0127$ &67 &92.35 \\
NDDR-CNN &
\colorbox{pos}{$+0.73$} &\colorbox{pos}{$+0.03$} &\colorbox{pos}{$+0.0086$} &\colorbox{pos}{$+0.0072$} &\colorbox{neg}{$-0.34$} &\colorbox{neg}{$-0.58$} &\colorbox{neg}{$-0.94$} &\colorbox{neg}{$-1.00$} &\colorbox{neg}{$-0.77$} &$+0.0066$ &67 &169.10 \\
AFA &
\colorbox{neg}{$-1.57$} &\colorbox{neg}{$-1.29$} &\colorbox{neg}{$-0.0073$} &\colorbox{neg}{$-0.0060$} &\colorbox{neg}{$-1.97$} &\colorbox{neg}{$-1.91$} &\colorbox{neg}{$-3.54$} &\colorbox{neg}{$-4.21$} &\colorbox{neg}{$-3.73$} &$-0.0507$ &0 &136.88 \\
RotoGrad &
\colorbox{pos}{$+0.98$} &\colorbox{pos}{$+0.05$} &\colorbox{pos}{$+0.0157$} &\colorbox{pos}{$+0.0062$} &\colorbox{neg}{$-0.79$} &\colorbox{neg}{$-1.01$} &\colorbox{neg}{$-2.63$} &\colorbox{neg}{$-2.56$} &\colorbox{neg}{$-1.77$} &$+0.0009$ &67 &75.03 \\
\cmidrule(lr){1-13}
MaxRoam &
\colorbox{pos}{$+1.42$} &\colorbox{pos}{$+0.41$} &\colorbox{pos}{$+0.0078$} &\colorbox{neg}{$-0.0068$} &\colorbox{pos}{$+0.36$} &\colorbox{pos}{$+0.51$} &\colorbox{neg}{$-2.03$} &\colorbox{neg}{$-1.55$} &\colorbox{neg}{$-2.46$} &$-0.0005$ &63 &71.89 \\
MTL-NAS &
\colorbox{pos}{$+0.81$} &\colorbox{pos}{$+0.01$} &\colorbox{pos}{$+0.0110$} &\colorbox{pos}{$+0.0102$} &\colorbox{neg}{$-0.15$} &\colorbox{neg}{$-0.52$} &\colorbox{neg}{$-0.46$} &\colorbox{neg}{$-0.22$} &\colorbox{neg}{$-2.11$} &$+0.0121$ &67 &183.40 \\
BMTAS &
\colorbox{pos}{$+0.73$} &\colorbox{pos}{$+0.22$} &\colorbox{pos}{$+0.0067$} &\colorbox{pos}{$+0.0017$} &\colorbox{pos}{$+0.01$} &\colorbox{neg}{$-0.03$} &\colorbox{neg}{$-0.00$} &\colorbox{neg}{$-0.08$} &\colorbox{neg}{$-0.00$} &$+0.0082$ &73 &116.04 \\
TSN &
\colorbox{neg}{$-0.87$} &\colorbox{neg}{$-0.75$} &\colorbox{neg}{$-0.0221$} &\colorbox{neg}{$-0.0067$} &\colorbox{pos}{$+0.50$} &\colorbox{pos}{$+0.30$} &\colorbox{pos}{$+1.35$} &\colorbox{pos}{$+1.22$} &\colorbox{pos}{$+0.90$} &$-0.0167$ &33 &50.58 \\
LTB &
\colorbox{pos}{$+0.31$} &\colorbox{pos}{$+0.23$} &\colorbox{pos}{$+0.0052$} &\colorbox{pos}{$+0.0096$} &\colorbox{pos}{$+0.02$} &\colorbox{neg}{$-0.22$} &\colorbox{neg}{$-0.86$} &\colorbox{neg}{$-0.33$} &\colorbox{pos}{$+0.04$} &$+0.0119$ &80 &86.85 \\
\cmidrule(lr){1-13}
DSMTL-IL &
\colorbox{pos}{$+0.60$} &\colorbox{pos}{$+0.30$} &\colorbox{pos}{$+0.0007$} &\colorbox{pos}{$+0.0005$} &\colorbox{pos}{$+0.26$} &\colorbox{pos}{$+0.06$} &\colorbox{pos}{$+0.14$} &\colorbox{pos}{$+0.33$} &\colorbox{pos}{$+0.41$} &$+0.0067$ &100 &142.41 \\
DSMTL-JL &
\colorbox{pos}{$+0.75$} &\colorbox{pos}{$+0.36$} &\colorbox{pos}{$+0.0114$} &\colorbox{pos}{$+0.0051$} &\colorbox{pos}{$+0.48$} &\colorbox{pos}{$+0.47$} &\colorbox{pos}{$+1.27$} &\colorbox{pos}{$+1.00$} &\colorbox{pos}{$+0.73$} &$+0.0221$ &100 &142.41 \\
DSMTL-AL &
\colorbox{pos}{$+1.25$} &\colorbox{pos}{$+0.71$} &\colorbox{pos}{$+0.0144$} &\colorbox{pos}{$+0.0070$} &\colorbox{pos}{$+0.45$} &\colorbox{pos}{$+0.40$} &\colorbox{pos}{$+0.83$} &\colorbox{pos}{$+0.99$} &\colorbox{pos}{$+0.74$} &$+0.0281$ &100 &93.96 \\
\bottomrule
\end{tabular}
}
\label{tab:nyuv2}
\end{table*}

\begin{table*}[!htb]
\centering
\caption{Performance of various models on the PASCAL-Context validation dataset. $\uparrow$ ($\downarrow$) indicates the higher (lower) the result, the better the performance. The green color indicates that the corresponding method performs better than the STL method and the red color indicates oppositely. The number of parameters (abbreviated as Parms.) is calculated in MB.}
\vskip -0.1in
\resizebox{\textwidth}{!}{
\begin{tabular}{lcccccccccccc}
\toprule
\multirow{3}{*}{Method} & \multicolumn{1}{c}{Segmentation} & \multicolumn{1}{c}{Human Parts} &
\multicolumn{2}{c}{Saliency} & \multicolumn{5}{c}{Surface Normal} & \multirow{3}{*}{$\Delta_I \uparrow$} &  \multirow{3}{*}{${\eta}\uparrow$} & \multirow{3}{*}{Parms. (M)$\downarrow$} \\
\cmidrule(lr){2-2} \cmidrule(lr){3-3} \cmidrule(lr){4-5} \cmidrule(l){6-10}
&\multirow{2.5}{*}{mIoU${\uparrow}$} & \multirow{2.5}{*}{mIoU${\uparrow}$} &
\multirow{2.5}{*}{mIoU${\uparrow}$} &
\multirow{2.5}{*}{maxF${\uparrow}$} &
\multicolumn{2}{c}{Angle Distance $\downarrow$} & \multicolumn{3}{c}{Within $t^{\circ}$ $\uparrow$} \\
\cmidrule(lr){6-7} \cmidrule(l){8-10} & & & & & Mean & Median  & 11.25 & 22.5 & 30 \\
\cmidrule(lr){1-13}
STL & $65.14$ & $58.58$ & $65.02$ & $77.47$ & $15.94$ & $24.87$ & $48.42$ & $80.79$ & $90.03$ & 0 & - & 63.60 \\
\cmidrule(lr){1-13}
HPS & \colorbox{neg}{$-0.37$} & \colorbox{neg}{$-0.67$} & \colorbox{neg}{$-0.92$} & \colorbox{neg}{$-0.51$} & \colorbox{neg}{$-1.73$} & \colorbox{neg}{$-1.29$} & \colorbox{neg}{$-6.43$} & \colorbox{neg}{$-4.86$} & \colorbox{neg}{$-3.02$} & $-0.0262$ & $0$ & 30.07 \\
Cross-stitch  & \colorbox{neg}{$-0.17$} & \colorbox{pos}{$+0.05$} & \colorbox{neg}{$-0.56$} & \colorbox{neg}{$-0.40$} & \colorbox{neg}{$-0.62$} & \colorbox{neg}{$-0.45$} & \colorbox{neg}{$-2.36$} & \colorbox{neg}{$-2.68$} & \colorbox{neg}{$-1.04$} & $-0.0096$ & $25$ & 79.46 \\
MTAN  & \colorbox{neg}{$-0.58$} & \colorbox{pos}{$+0.50$} & \colorbox{neg}{$-0.45$} & \colorbox{neg}{$-0.23$} & \colorbox{neg}{$-1.20$} & \colorbox{neg}{$-0.89$} & \colorbox{neg}{$-4.58$} & \colorbox{neg}{$-3.31$} & \colorbox{neg}{$-2.04$} & $-0.0147$ & $25$ & 36.61 \\
NDDR-CNN  & \colorbox{pos}{$+0.14$} & \colorbox{pos}{$+0.60$} & \colorbox{pos}{$+0.07$} & \colorbox{pos}{$+0.00$} & \colorbox{neg}{$-0.37$} & \colorbox{neg}{$-0.24$} & \colorbox{neg}{$-1.50$} & \colorbox{neg}{$-0.94$} & \colorbox{neg}{$-0.59$} & $-0.0008$ & $75$ & 69.25  \\
AFA & \colorbox{pos}{$+{2.12}$} & \colorbox{pos}{$+{2.11}$} & \colorbox{neg}{$-1.95$} & \colorbox{neg}{$-3.96$} & \colorbox{neg}{$-1.63$} & \colorbox{neg}{$-1.28$} & \colorbox{neg}{$-5.68$} & \colorbox{neg}{$-4.42$} & \colorbox{neg}{$-2.88$} & $-0.0108$ & $50$ & 199.1 \\
RotoGrad &
\colorbox{neg}{$-1.57$} &\colorbox{neg}{$-0.23$} &\colorbox{pos}{$+0.34$} &\colorbox{pos}{$+0.35$} &\colorbox{pos}{$+0.11$} &\colorbox{pos}{$+0.05$} &\colorbox{pos}{$+0.14$} &\colorbox{pos}{$+0.13$} &\colorbox{pos}{$+0.07$} &$-0.0051$ &50 &33.22 \\
\cmidrule(lr){1-13}
MaxRoam &
\colorbox{neg}{$-1.33$} &\colorbox{pos}{$+0.71$} &\colorbox{neg}{$-0.99$} &\colorbox{neg}{$-3.42$} &\colorbox{pos}{$+0.16$} &\colorbox{pos}{$+0.10$} &\colorbox{pos}{$+0.35$} &\colorbox{pos}{$+0.23$} &\colorbox{pos}{$+0.15$} &$-0.0082$ &50 &30.07 \\
MTL-NAS &
\colorbox{neg}{$-0.68$} &\colorbox{pos}{$+0.18$} &\colorbox{neg}{$-0.67$} &\colorbox{neg}{$-0.20$} &\colorbox{pos}{$+0.68$} &\colorbox{pos}{$+0.47$} &\colorbox{pos}{$+2.56$} &\colorbox{pos}{$+1.58$} &\colorbox{pos}{$+0.87$} &$+0.0037$ &50 &41.24 \\
BMTAS &
\colorbox{neg}{$-0.14$} &\colorbox{pos}{$+0.39$} &\colorbox{neg}{$-0.36$} &\colorbox{neg}{$-0.15$} &\colorbox{neg}{$-0.06$} &\colorbox{neg}{$-0.06$} &\colorbox{neg}{$-0.34$} &\colorbox{neg}{$-0.29$} &\colorbox{neg}{$-0.10$} &-0.0006 &12 &45.28 \\
TSN &
\colorbox{pos}{$+1.63$} &\colorbox{neg}{$-0.67$} &\colorbox{pos}{$+0.08$} &\colorbox{neg}{$-0.51$} &\colorbox{neg}{$-0.62$} &\colorbox{neg}{$-0.29$} &\colorbox{neg}{$-5.43$} &\colorbox{neg}{$-3.86$} &\colorbox{neg}{$-2.02$} &-0.0089 &25 & 26.85 \\
LTB &
\colorbox{neg}{$-0.65$} &\colorbox{pos}{$+1.55$} &\colorbox{neg}{$-0.93$} &\colorbox{neg}{$-3.59$} &\colorbox{pos}{$+0.58$} &\colorbox{pos}{$+0.39$} &\colorbox{pos}{$+2.03$} &\colorbox{pos}{$+1.36$} &\colorbox{pos}{$+0.74$} &+0.0025 &75 &62.60 \\
\cmidrule(lr){1-13}
DSMTL-IL &
\colorbox{pos}{$+0.01$} & \colorbox{pos}{$+1.05$} & \colorbox{pos}{$+0.20$} & \colorbox{pos}{$+0.13$} & \colorbox{pos}{$+0.26$} & \colorbox{pos}{$+0.22$} & \colorbox{pos}{$+1.08$} & \colorbox{pos}{$+0.74$} & \colorbox{pos}{$+0.39$} & $+0.0082$ & $100$ & 74.78 \\
DSMTL-JL &
\colorbox{pos}{$+0.18$} &\colorbox{pos}{$+1.87$} &\colorbox{pos}{$+0.44$} &\colorbox{pos}{$+0.47$} &\colorbox{pos}{$+0.41$} &\colorbox{pos}{$+0.27$} &\colorbox{pos}{$+1.21$} &\colorbox{pos}{$+0.86$} &\colorbox{pos}{$+0.50$} &$+0.0142$ &100 & 74.78 \\
DSMTL-AL &
\colorbox{pos}{$+0.82$} &\colorbox{pos}{$+1.28$} &\colorbox{pos}{$+0.37$} &\colorbox{pos}{$+0.31$} &\colorbox{pos}{$+0.04$} &\colorbox{pos}{$+0.02$} &\colorbox{pos}{$+0.01$} &\colorbox{pos}{$+0.01$} &\colorbox{pos}{$+0.09$} &$+0.0106$ &100 & 63.15 \\
\bottomrule
\end{tabular}
}
\label{tab:pascal}
\end{table*}

\begin{table*}[!htb]
\centering
\caption{Performance of various models on the Taskonomy validation dataset. $\uparrow$ ($\downarrow$) indicates the higher (lower) the result, the better the performance. The green color indicates that the corresponding method performs better than the STL method and the red color indicates oppositely. The number of parameters (abbreviated as Parms.) is calculated in MB.}
\vskip -0.1in
\resizebox{\textwidth}{!}{
\begin{tabular}{lcccccccccccccc}
\toprule
  \multirow{3}{*}{Method} & \multicolumn{2}{c}{Segmentation} & \multicolumn{2}{c}{Depth} & Keypoints & Edges & \multicolumn{5}{c}{Surface Normal} & \multirow{3}{*}{$\Delta_I \uparrow$} & \multirow{3}{*}{${\eta}\uparrow$} & \multirow{3}{*}{Parms. (M)$\downarrow$} \\
 \cmidrule(r){2-3} \cmidrule(r){4-5} \cmidrule(r){6-6} \cmidrule(r){7-7} \cmidrule(r){8-12}
 & \multirow{2}{*}{mIoU $\uparrow$} &  \multirow{2}{*}{Pix Err $\downarrow$} &  \multirow{2}{*}{Abs Err $\downarrow$} &  \multirow{2}{*}{Rel Err$\downarrow$} & \multirow{2}{*}{Abs Err $\downarrow$} & \multirow{2}{*}{Abs Err $\downarrow$} &  \multicolumn{2}{c}{Angle Distance $\downarrow$} & \multicolumn{3}{c}{Within $t^{\circ}$ $\uparrow$}  \\ \cmidrule(r){8-9} \cmidrule(r){10-12} & & & &  & & & Mean  & Median   & 11.25  & 22.5  & 30  \\
\cmidrule(lr){1-15}
STL & $65.42$ & $2.37$ & $0.0072$ & $0.0117$ & $0.1103$ & $0.1349$ & $10.39$ & $4.19$ & $73.67$ & $86.21$ & $90.52$ & 0 & - & 79.50 \\
\cmidrule(lr){1-15}
HPS &
\colorbox{pos}{$+0.33$} &\colorbox{pos}{$+0.05$} &\colorbox{neg}{$-0.0011$} &\colorbox{neg}{$-0.0018$} &\colorbox{pos}{$+0.0015$} &\colorbox{neg}{$-0.0004$} &\colorbox{neg}{$-0.69$} &\colorbox{neg}{$-0.45$} &\colorbox{neg}{$-1.94$} &\colorbox{neg}{$-1.09$} &\colorbox{neg}{$-0.79$} &$-0.0348$ &40 &34.79 \\
Cross-stitch &
\colorbox{pos}{$+0.87$} &\colorbox{pos}{$+0.67$} &\colorbox{pos}{$+0.0005$} &\colorbox{pos}{$+0.0009$} &\colorbox{pos}{$+0.0100$} &\colorbox{pos}{$+0.0019$} &\colorbox{pos}{$+1.50$} &\colorbox{pos}{$+0.17$} &\colorbox{pos}{$+0.42$} &\colorbox{pos}{$+0.38$} &\colorbox{pos}{$+0.30$} &$+0.0603$ &100 &79.46 \\
MTAN &
\colorbox{pos}{$+0.34$} &\colorbox{pos}{$+0.69$} &\colorbox{neg}{$-0.0011$} &\colorbox{neg}{$-0.0019$} &\colorbox{neg}{$-0.0248$} &\colorbox{neg}{$-0.0111$} &\colorbox{pos}{$+0.99$} &\colorbox{neg}{$-0.18$} &\colorbox{pos}{$+2.30$} &\colorbox{pos}{$+2.89$} &\colorbox{pos}{$+2.41$} &$+0.0240$ &36 &36.61 \\
NDDR-CNN &
\colorbox{pos}{$+0.64$} &\colorbox{pos}{$+0.69$} &\colorbox{pos}{$+0.0002$} &\colorbox{neg}{$-0.0028$} &\colorbox{pos}{$+0.0133$} &\colorbox{pos}{$+0.0043$} &\colorbox{pos}{$+1.86$} &\colorbox{pos}{$+0.47$} &\colorbox{pos}{$+0.48$} &\colorbox{pos}{$+0.42$} &\colorbox{pos}{$+0.33$} &$+0.0593$ &90 &88.32 \\
AFA &
\colorbox{pos}{$+0.84$} &\colorbox{pos}{$+0.57$} &\colorbox{neg}{$-0.0016$} &\colorbox{neg}{$-0.0027$} &\colorbox{pos}{$+0.0465$} &\colorbox{pos}{$+0.0524$} &\colorbox{pos}{$+1.08$} &\colorbox{neg}{$-0.17$} &\colorbox{pos}{$+0.30$} &\colorbox{pos}{$+0.33$} &\colorbox{pos}{$+0.27$} &$+0.0476$ &76 &242.1 \\
RotoGrad &
\colorbox{pos}{$+0.01$} &\colorbox{neg}{$-0.03$} &\colorbox{pos}{$+0.0012$} &\colorbox{pos}{$+0.0019$} &\colorbox{pos}{$+0.0028$} &\colorbox{pos}{$+0.0003$} &\colorbox{pos}{$+0.13$} &\colorbox{pos}{$+0.27$} &\colorbox{pos}{$+0.20$} &\colorbox{pos}{$+0.05$} &\colorbox{pos}{$+0.02$} &$+0.0213$ &90 &37.94 \\
\cmidrule(lr){1-15}
MaxRoam &
\colorbox{neg}{$-0.61$} &\colorbox{neg}{$-0.10$} &\colorbox{pos}{$+0.0013$} &\colorbox{pos}{$+0.0021$} &\colorbox{pos}{$+0.0028$} &\colorbox{pos}{$+0.0015$} &\colorbox{pos}{$+0.06$} &\colorbox{pos}{$+0.22$} &\colorbox{pos}{$+0.06$} &\colorbox{neg}{$-0.04$} &\colorbox{neg}{$-0.05$} &$+0.0164$ &72 &34.79 \\
MTL-NAS &
\colorbox{pos}{$+0.01$} &\colorbox{neg}{$-0.04$} &\colorbox{pos}{$+0.0014$} &\colorbox{pos}{$+0.0022$} &\colorbox{pos}{$+0.0027$} &\colorbox{pos}{$+0.0014$} &\colorbox{pos}{$+0.16$} &\colorbox{pos}{$+0.30$} &\colorbox{pos}{$+0.35$} &\colorbox{pos}{$+0.13$} &\colorbox{pos}{$+0.04$} &$+0.0253$ &90 &45.97 \\
BMTAS &
\colorbox{neg}{$-0.53$} &\colorbox{neg}{$-0.10$} &\colorbox{pos}{$+0.0019$} &\colorbox{pos}{$+0.0031$} &\colorbox{pos}{$+0.0048$} &\colorbox{pos}{$+0.0005$} &\colorbox{pos}{$+0.03$} &\colorbox{pos}{$+0.17$} &\colorbox{pos}{$+0.01$} &\colorbox{neg}{$-0.08$} &\colorbox{neg}{$-0.10$} &$+0.0217$ &72 &66.27 \\
TSN &
\colorbox{neg}{$-0.07$} &\colorbox{neg}{$-0.02$} &\colorbox{pos}{$+0.0012$} &\colorbox{pos}{$+0.0019$} &\colorbox{pos}{$+0.0020$} &\colorbox{pos}{$+0.0001$} &\colorbox{pos}{$+0.03$} &\colorbox{pos}{$+0.24$} &\colorbox{neg}{$-0.12$} &\colorbox{neg}{$-0.20$} &\colorbox{neg}{$-0.19$} &$+0.0184$ &68 &26.85 \\
LTB &
\colorbox{pos}{$+0.30$} &\colorbox{pos}{$+0.01$} &\colorbox{pos}{$+0.0015$} &\colorbox{pos}{$+0.0024$} &\colorbox{pos}{$+0.0038$} &\colorbox{pos}{$+0.0001$} &\colorbox{pos}{$+0.46$} &\colorbox{pos}{$+0.69$} &\colorbox{pos}{$+0.80$} &\colorbox{pos}{$+0.23$} &\colorbox{pos}{$+0.13$} &$+0.0426$ &100 &43.19 \\
\cmidrule(lr){1-15}
DSMTL-IL &
\colorbox{pos}{$+0.67$} &\colorbox{pos}{$+0.01$} &\colorbox{pos}{$+0.0034$} &\colorbox{pos}{$+0.0055$} &\colorbox{pos}{$+0.0056$} &\colorbox{pos}{$+0.0014$} &\colorbox{pos}{$+0.97$} &\colorbox{pos}{$+1.06$} &\colorbox{pos}{$+2.00$} &\colorbox{pos}{$+0.95$} &\colorbox{pos}{$+0.65$} &$+0.0844$ &100 &90.68 \\
DSMTL-JL &
\colorbox{pos}{$+0.89$} &\colorbox{pos}{$+0.03$} &\colorbox{pos}{$+0.0035$} &\colorbox{pos}{$+0.0057$} &\colorbox{pos}{$+0.0056$} &\colorbox{pos}{$+0.0014$} &\colorbox{pos}{$+0.98$} &\colorbox{pos}{$+1.08$} &\colorbox{pos}{$+2.00$} &\colorbox{pos}{$+0.95$} &\colorbox{pos}{$+0.66$} &$+0.0872$ &100 &90.68 \\
DSMTL-AL &
\colorbox{pos}{$+0.45$} &\colorbox{pos}{$+0.02$} &\colorbox{pos}{$+0.0012$} &\colorbox{pos}{$+0.0019$} &\colorbox{pos}{$+0.0038$} &\colorbox{pos}{$+0.0013$} &\colorbox{pos}{$+0.59$} &\colorbox{pos}{$+0.83$} &\colorbox{pos}{$+1.11$} &\colorbox{pos}{$+0.39$} &\colorbox{pos}{$+0.22$} &$+0.0455$ &100 &54.36 \\
\bottomrule
\end{tabular}
}
\label{tab:taskonomy}
\end{table*}

\subsection{Experimental Setup}
\label{sec:experimental_setup}

The baseline methods in comparison include the Single-Task Learning (STL) that trains each task separately, popular MTL architectures including the HPS model that adopts the multi-head hard sharing architecture,  Cross-stitch \cite{misra2016cross}, MTAN \cite{liu2019end}, NDDR-CNN \cite{gao2019nddr}, RotoGrad \cite{javaloy2022rotograd}, and AFA \cite{cui2021adaptive}, and popular architecture learning methods for MTL such as
MaxRoam \cite{MaxRoam}, TSN \cite{sun2021task},
MTL-NAS \cite{MTL-NAS}, BMTAS \cite{BMTAS}, and LTB \cite{LTB}. For fair comparison, we use the same backbone for all the models.
Similar to \cite{liu2019end},
we use the Deeplab-ResNet \cite{chen2017deeplab} with atrous convolutions as encoders and the ASPP architecture \cite{chen2017deeplab} as decoders.
We adopt the ResNet-50 pretrained on ImageNet for the CityScapes and NYUv2 datasets to implement the the Deeplab-ResNet, and use the pretrained ResNet-18 on the larger PASCAL-Context and Taskonomy datasets for training efficiency. We use the cross-entropy loss for the semantic segmentation, human parts segmentation and saliency estimation tasks, the cosine similarity loss for the surface normal prediction task, and the $L_1$ loss for other tasks. For the DSMTL-AL method, a module is defined as a layer in the Deeplab-ResNet model and the branch position is set before and after each layer. Therefore, we have five layers (including conv1) and have six branch positions. For optimization, we use the Adam method \cite{kingma2014adam} with the learning rate as $10^{-4}$.
All the experiments are conducted on Tesla V100 GPUs.


\subsection{Experimental Results}

Tables \ref{tab:cityscapes}-\ref{tab:taskonomy} show the performance of all the models in comparison on different datasets.
On the CityScapes dataset, the proposed DSMTL-IL, DSMTL-JL, DSMTL-AL and some baseline methods (i.e., Cross-stitch, MTAN, RotoGrad, BMTAS, TSN, and NDDR-CNN) perform no worse than the STL model on each metric and hence under such setting, they achieve \textit{safe multi-task learning} (i.e., ${\eta} = 100$). In addition, the proposed DSMTL-AL model achieves the best $\Delta_I$, which demonstrates its effectiveness.
On the NYUv2 and PASCAL-Context datasets, none of the baselines can achieve \textit{safe multi-task learning}, while the proposed methods (i.e., DSMTL-IL, DSMTL-JL, and DSMTL-AL) can achieve that, which again shows the effectiveness of the proposed methods.
On the Taskonomy dataset, only the cross-stitch network, LTB and the proposed methods can achieve \textit{safe multi-task learning} and among them, the proposed DSMTL-JL method performs the best in terms of $\Delta_I$. According to results shown in Table \ref{tab:taskonomy}, we can see that the AFA method achieves the best performance on the keypoint detection and edge detection tasks for the Taskonomy dataset, but it does not achieve \textit{safe multi-task learning}, which is the focus of the proposed methods. Moreover, the number of parameters in the AFA model is 2.66 times over that of the DSMTL-IL and DSMTL-JL models, which may explain the improvement of the AFA model on some tasks.


For the proposed three methods, all of them achieve safeness on the four datasets, which demonstrates their effectiveness.
The proposed DSMTL-JL method performs better than the DSMTL-IL method on the NYUv2, PASCAL-Context, and Taskonomy datasets in terms of $\Delta_I$.
One reason is that the DSMTL-JL model can learn a better model by optimizing all the model parameters together, while the DSMTL-IL model adopts a two-stage optimization strategy.
Compared with those two models, the DSMTL-AL method can achieve a better trade-off between the performance and the model size as it performs comparable or even better than the better one in DSMTL-IL and DSMTL-JL and its model size is much smaller than those in DSMTL-IL and DSMTL-JL, which matches the design goal of the DSMTL-AL method.

\begin{table*}[!htbp]
\centering
\caption{$\{\alpha_t\}$ learned in the DSMTL models as well as branch points and $\{w_t\}$ learned by the DSMTL-AL method on four CV datasets. `SS' stands for the semantic segmentation task, `DE' denotes the depth estimation task, `SNP' is for the surface normal prediction task, `HPS' corresponds to the human parts segmentation task, `SE' stands for the saliency estimation task, `KD' stands for the keypoint detection task, and `ED' denotes the edge detection task.}
\vskip -0.1in
\resizebox{\linewidth}{!}{
\begin{tabular}{l|cc|ccc|cccc|ccccc}
\toprule \multirow{2}{*}{{Method}} & \multicolumn{2}{c|}{{CityScapes}} & \multicolumn{3}{c|}{{NYUv2}} & \multicolumn{4}{c|}{{PASCAL-Context}} & \multicolumn{5}{c}{{Taskonomy}} \\
\cmidrule(lr){2-3}
\cmidrule(lr){4-6}
\cmidrule(lr){7-10}
\cmidrule(lr){11-15}
& SS & DE & SS & DE & SNP & SS & HPS & SE & SNP & SE & DE & KD & ED & SNP \\
\midrule
DSMTL-IL & $0.5539$ & $0.4470$ & $0.5624$ & $0.5083$ & $0.4745$ & $0.3823$ & $0.3899$ & $0.4426$ & $0.3100$ & $0.4256$ & $0.3162$ & $0.3143$ & $0.3768$ & $0.3056$ \\
DSMTL-JL & $0.5002$ & $0.4960$ & $0.4383$ & $0.5188$ & $\boxed{0.1997}$ & $0.4739$ & $0.5529$ & $0.3701$ & $\boxed{0.2304}$ & $0.4886$ & $0.4565$ & $0.4504$ & $0.4578$ & $\boxed{0.2584}$ \\
\midrule
DSMTL-AL
 & $-$ & $0.3675$ & $-$ & $-$ & $0.3931$ & $-$ & $0.5651$ & $0.5545$ & $0.3952$ & $-$ & $-$ & $0.3565$ & $-$ & $0.3438$ \\
- branch point & $p=6$ & $p=4$ & $p=6$ & $p=6$ & $p=4$ & $p=6$ & $p=2$ & $p=1$ & $p=4$ & $p=6$ & $p=6$ & $p=4$ & $p=6$ & $p=1$ \\
- learned $w_t$ & $0.3274$ & $0.6726$ & $0.0283$ & $0.0568$ & $0.9149$ & $0.0574$ & $0.0649$ & $0.8032$ & $0.0745$ & $0.0012$ & $0.0033$ & $0.4932$ & $0.4932$ & $0.0091$ \\
\bottomrule
\end{tabular}
}
\label{tab:alpha}
\end{table*}

\begin{figure}[!htbp]
\centering
\includegraphics[width=1\linewidth]{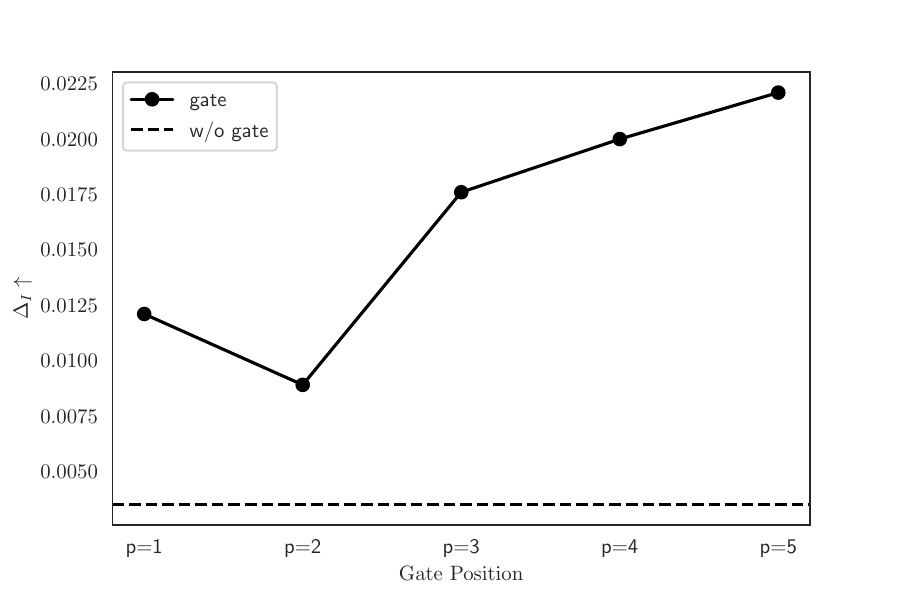}
\caption{The performance of the DSMTL-JL model on the NYUv2 dataset when varying the position of the gates, where $p$ represents the position of the gate.
}
\label{fig:gate_position}
\end{figure}

\subsection{Analysis on the Position of Gate}

In this section, we study how the position of the gate affects the performance of the proposed models and we use the DSMTL-JL model as an example.
As there are $P-1$ possible positions for each gate by excluding the position before the first layer, we try each of them to see which one is the best in terms of the performance. To avoid the exponential complexity, we assume that gate positions of different tasks are the same.
Specifically, we put a gate $g_t(\cdot,p)$ after the $p$-th module of the private encoder in task $t$ and the entire encoder for task $t$ is formulated as
\begin{equation*}
g_t(\mathbf{x},p)=
f_t(\alpha^p_t f_S(\mathbf{x},p) + (1-\alpha^p_t) f_t(\mathbf{x},1,p),p+1,P-1),
\end{equation*}
where as defined in Section \ref{sec:DSMTL_AL}, $f_S(\cdot,p)$ denotes the output of the $p$-th module in $f_S$ and $f_t(\cdot,p,q)$ denotes the output of the $q$-th module in $f_t$ starting from its $p$-th module.

\begin{table}[h]
\caption{Ablation study of the DSMTL models on the NYUv2 dataset.
}
\centering
\resizebox{0.8\linewidth}{!}{
\begin{tabular}{lccc}
\toprule
Method & $\Delta_I\uparrow$ &${\eta}\uparrow$ \\
\midrule
DSMTL-IL &$+0.0067$ & 100 \\
-w/o learnable gate &$+0.0034$ &77  \\
\midrule
DSMTL-JL &$+0.0221$ & 100 \\
-w/o learnable gate &$+0.0035$ &83 \\
\midrule
DSMTL-AL &$+0.0281$ & 100 \\
-w/o learnable gate &$+0.0220$ &100 \\
-w/o learnable task weighting &$+0.0166$ &80 \\
\bottomrule
\end{tabular}
}
\vskip -0.15in
\label{tab:abs}
\end{table}

According to Figure \ref{fig:gate_position}, when changing the gate position from $p=2$ to $p=5$, the performance of the DMTL-JL model becomes better, and $p=5$ gives the best performance, which justifies the choice of the gate position in the DSMTL-IL and DSMTL-JL models and also inspires the design of the final encoder in the DSMTL-AL method as defined in Eq.~\eqref{DSMTL_AL_final_encoder}.

\subsection{Analysis on Learned Task Relevance}

We show the learned $\{\alpha_t\}$ of the proposed methods in Table \ref{tab:alpha}. According to the results, we can see that some $\alpha_t$'s are closed to $0.5$, which implies that in those cases, the public encoder and the private encoder are both important to the corresponding tasks. Thus, only using the public encoder (i.e., HPS) and only using the private encoder (i.e., STL) cannot achieve good performance, while the proposed models can take the advantages of these two methods to achieve better performance in most cases.
Moreover, some of the learned $\alpha_t$'s have relatively small values (i.e., values smaller than $0.3$), which are shown in box. These small values indicate that for the surface normal prediction task on the NYUv2, PASCAL-Context, and Taskonomy datasets, the public encoder is relatively unimportant, and this is consistent with the architecture learned by the DSMTL-AL method, where the surface normal prediction task branches out at the first several public modules and switches to the private modules. This may imply that the surface normal prediction task is not strongly related to other tasks on these datasets, which aligns with the task relationship founded in \cite{sun2021task}. On the other hand, this observation may explain why HPS performs much worse than STL and why the proposed  methods have good performance on those datasets (refer to Tables \ref{tab:nyuv2}-\ref{tab:taskonomy}).
For the DSMTL-AL method, we can see that in each dataset, at least one task choose to use the entire shared encoder, which corresponds to the choice of the branch point $p=6$ and hence indicates that there is no need to learn the corresponding $\alpha_t$ defined in Eq.~\eqref{DSMTL_AL_final_encoder}, and hence including the public encoder in the design of the final encoder defined in Eq.~\eqref{DSMTL_AL_final_encoder} for the DSMTL-AL method will not increase the model size.

As the loss scales and converge speed of different tasks vary a lot, using identical loss weighting (i.e., $w_t=\frac{1}{m}$ for $t=1,\ldots,m$) may lead to suboptimal performance.
The learned $w_t$'s in DSMTL-AL at the bottom of Table \ref{tab:alpha} show that loss weights are uneven. For example, the surface normal prediction task has a larger loss weight than the other two tasks in the NYUv2 dataset, while this task has a smaller loss weight than some other tasks in the PASCAL-Context and Taskonomy datasets, which indicates that the proposed DSMTL-AL method could learn adaptive loss weighting strategies in different datasets.

\subsection{Ablation Study}
In Table \ref{tab:abs}, we provide the ablation study for the proposed DSMTL models.
For ``w/o learnable gate'', we replace the learnable gates in Eq.~(\ref{eq:gate_fun}) or \eqref{DSMTL_AL_final_encoder} with simply compute the average of outputs of both public and private encoders. Compared with the original DSMTL methods, the performance of those variants degrades in terms of $\Delta_I$, which verifies the usefulness of the learnable gates. The safeness coefficient decreases in both DSMTL-IL and DSMTL-JL cases, which indicates that learnable gates are a key ingredient for the DSMTL-IL and DSMTL-JL models to achieve the safeness. The safeness coefficient still keeps as 100 in the variant of the DSMTL-AL method, which implies that the DSMTL-AL method can learn a reliable  architecture to achieve the safeness.

For ``w/o learning task weighting'', we replace the learned task weights in the DSMTL-AL method with identical loss weights during the retraining process. This variant without the learnable task weighting has inferior performance to the DSMTL-AL method, which indicates that learning task weighting in the DSMTL-AL method can not only reduce tedious costs to tune loss weights but also improve the performance of the DSMTL-AL method.

\subsection{Combination and Comparison with Loss Weighting Strategies}

The loss weighting scheme adopted in the DSMTL-IL and DSMTL-JL methods is the commonly used Equally Weighting (EW) strategy (i.e., all the loss weights are equal to $\frac{1}{m}$ in problems \eqref{eqa: obj_al} and \eqref{eqa: obj}). The loss weighting scheme adopted in the DSMTL-AL method is a Learnable Weighting (LW) strategy as in Eq.~(\ref{eq:smtlba_lw}). In this section, we show that the proposed DSMTL models could be combined with some loss weighting methods in MTL, including Uncertainty Weights (UW) \cite{kendall2018multi}, Dynamic Weight Average (DWA) \cite{liu2019end}, Geometric Loss Strategy (GLS) \cite{chennupati2019multinet++},
PCGrad \cite{PCGrad}, and CAGrad \cite{CAGrad}.

According to experimental results shown in Table \ref{tab:nyuv2_weighting}, the combination of the DSMTL-IL method and other methods than the EW strategy has inferior performance and a lower safeness coefficient $\eta$, which verifies the usefulness of the EW strategy adopted in the DSMTL-IL method. Differently, the performance of the DSMTL-JL method could be improved in terms of $\Delta_I$ when combining with some loss weighting strategies (i.e., GLS and CAGrad), and all the combinations have the largest safeness coefficients.
One reason for the aforementioned difference between the DSMTL-IL and DSMTL-JL methods is that the DSMTL-IL method fixes all parameters of the private encoders and decoders in the second stage of the training process, making the corresponding parameters not fully updated and therefore leading to the inferior performance.
For the DSMTL-AL method, replacing LW with \mbox{PCGrad} can further improve the performance while other loss weighting strategies degrade the performance.

In terms of computational cost or the training speedup over the STL, the PCGrad and CAGrad methods have large computational overhead since they need to project huge-dimensional gradients of different tasks on each training step. The UW, DWA and GLS methods have relatively small computational overhead but with limited performance improvement.
The LW strategy in DSMTL-AL has no computational overhead during the retraining process and have competitive performance compared with various loss weighting strategy, which demonstrate the effectiveness of the LW strategy.

\begin{table}[!htb]
\centering
\caption{Combining DSMTL models with various loss weighting strategies on the NYUv2 validation dataset.}
\vskip -0.1in
\resizebox{\linewidth}{!}{
\begin{tabular}{lcccc}
\toprule
Architecture & Methods & Train Speedup $\uparrow$ & $\Delta_I\uparrow$ & ${\eta}\uparrow$\\
\midrule
STL &- &1.0x &0 &100 \\
\midrule
\multirow{6}{*}{HPS}
&EW &1.69x &+0.0001 &67 \\
&UW &1.66x &-0.0040 &67 \\
&DWA &1.68x &+0.0029 &67 \\
&GLS &1.68x &\textbf{+0.0158} &67 \\
&PCGrad &0.78x &+0.0037 &67 \\
&CAGrad &0.61x &+0.0023 &67 \\
\midrule
\multirow{6}{*}{DSMTL-IL}
&EW &1.18x &\textbf{+0.0067} &100 \\
&UW &1.15x &+0.0035 &53 \\
&DWA &1.17x &+0.0038 &60 \\
&GLS &1.16x &+0.0041 &67 \\
&PCGrad &0.43x &-0.0397 &0 \\
&CAGrad &0.36x &-0.0399 &0 \\
\midrule
\multirow{6}{*}{DSMTL-JL}
&EW &0.89x &+0.0221 &100 \\
&UW &0.83x &+0.0134 &100 \\
&DWA &0.84x &+0.0191 &100 \\
&GLS &0.88x &\textbf{+0.0269} &100 \\
&PCGrad &0.34x &+0.0164 &100 \\
&CAGrad &0.30x &+0.0237 &100 \\
\midrule
\multirow{7}{*}{DSMTL-AL}
&LW &1.18x &+0.0281 &100 \\
&EW &1.18x &+0.0166 &100 \\
&UW &1.14x &+0.0163 &100 \\
&DWA &1.16x  &+0.0204 &93 \\
&GLS &1.17x  &+0.0203 &80 \\
&PCGrad &0.42x &\textbf{+0.0299} &100 \\
&CAGrad &0.38x &+0.0243 &100 \\
\bottomrule
\end{tabular}
}
\vskip -0.1in
\label{tab:nyuv2_weighting}
\end{table}

\section{Conclusion}

In this paper, we formally define the problem of \textit{safe multi-task learning}, and propose a simple and effective DSMTL method that can learn to combine the shared and task-specific representations. We theoretically analyze the proposed models and prove that the proposed DSMTL-IL and DSMTL-JL methods are guarantee to achieve some versions of \textit{safe multi-task learning}.
To solve the scalability issue of the proposed DSMTL-IL and DSMTL-JL methods, we further propose the DMSTL-AL method to learn a compact architecture via techniques in neural architecture search. Extensive experiments demonstrate the effectiveness of the proposed methods.
In the future work, we are interested in generalizing the DSMTL methods to other learning problems.

\section*{Acknowledgements}

This work is supported by NSFC key grant under grant no. 62136005, NSFC general grant under grant no. 62076118, and Shenzhen fundamental research program JCYJ20210324105000003.

\bibliographystyle{plain}
\bibliography{DSMTL}

\begin{thebibliography}{10}

\bibitem{bartlett2002rademacher}
Peter~L Bartlett and Shahar Mendelson.
\newblock Rademacher and gaussian complexities: Risk bounds and structural
  results.
\newblock {\em JMLR}, 2002.

\bibitem{bragman2019stochastic}
Felix~JS Bragman, Ryutaro Tanno, Sebastien Ourselin, Daniel~C Alexander, and
  Jorge Cardoso.
\newblock Stochastic filter groups for multi-task cnns: Learning specialist and
  generalist convolution kernels.
\newblock In {\em ICCV}, 2019.

\bibitem{BMTAS}
David Bruggemann, Menelaos Kanakis, Stamatios Georgoulis, and Luc~Van Gool.
\newblock Automated search for resource-efficient branched multi-task networks.
\newblock In {\em BMVC}, 2020.

\bibitem{cao2018partially}
Jiajiong Cao, Yingming Li, and Zhongfei Zhang.
\newblock Partially shared multi-task convolutional neural network with local
  constraint for face attribute learning.
\newblock In {\em CVPR}, 2018.

\bibitem{caruana97}
R.~Caruana.
\newblock Multitask learning.
\newblock {\em Machine Learning}, 1997.

\bibitem{chen2017deeplab}
Liang-Chieh Chen, George Papandreou, Iasonas Kokkinos, Kevin Murphy, and Alan~L
  Yuille.
\newblock Deeplab: Semantic image segmentation with deep convolutional nets,
  atrous convolution, and fully connected crfs.
\newblock {\em IEEE TPAMI}, 2017.

\bibitem{chennupati2019multinet++}
Sumanth Chennupati, Ganesh Sistu, Senthil Yogamani, and Samir A~Rawashdeh.
\newblock Multinet++: Multi-stream feature aggregation and geometric loss
  strategy for multi-task learning.
\newblock In {\em CVPR Workshops}, 2019.

\bibitem{cordts2016cityscapes}
Marius Cordts, Mohamed Omran, Sebastian Ramos, Timo Rehfeld, Markus Enzweiler,
  Rodrigo Benenson, Uwe Franke, Stefan Roth, and Bernt Schiele.
\newblock The cityscapes dataset for semantic urban scene understanding.
\newblock In {\em CVPR}, 2016.

\bibitem{cui2021adaptive}
Chaoran Cui, Zhen Shen, Jin Huang, Meng Chen, Mingliang Xu, Meng Wang, and
  Yilong Yin.
\newblock Adaptive feature aggregation in deep multi-task convolutional neural
  networks.
\newblock {\em IEEE TCSVT}, 2021.

\bibitem{evgeniou2004regularized}
Theodoros Evgeniou and Massimiliano Pontil.
\newblock Regularized multi--task learning.
\newblock In {\em KDD}, 2004.

\bibitem{franceschi2018bilevel}
Luca Franceschi, Paolo Frasconi, Saverio Salzo, Riccardo Grazzi, and
  Massimiliano Pontil.
\newblock Bilevel programming for hyperparameter optimization and
  meta-learning.
\newblock In {\em ICML}, 2018.

\bibitem{MTL-NAS}
Yuan Gao, Haoping Bai, Zequn Jie, Jiayi Ma, Kui Jia, and Wei Liu.
\newblock {MTL-NAS:} task-agnostic neural architecture search towards
  general-purpose multi-task learning.
\newblock In {\em CVPR}, 2020.

\bibitem{gao2019nddr}
Yuan Gao, Jiayi Ma, Mingbo Zhao, Wei Liu, and Alan~L Yuille.
\newblock Nddr-cnn: Layerwise feature fusing in multi-task cnns by neural
  discriminative dimensionality reduction.
\newblock In {\em CVPR}, 2019.

\bibitem{guo2020safe}
Lan-Zhe Guo, Zhen-Yu Zhang, Yuan Jiang, Yu-Feng Li, and Zhi-Hua Zhou.
\newblock Safe deep semi-supervised learning for unseen-class unlabeled data.
\newblock In {\em ICML}, 2020.

\bibitem{guo2020learning}
Pengsheng Guo, Chen-Yu Lee, and Daniel Ulbricht.
\newblock Learning to branch for multi-task learning.
\newblock In {\em ICML}, 2020.

\bibitem{LTB}
Pengsheng Guo, Chen{-}Yu Lee, and Daniel Ulbricht.
\newblock Learning to branch for multi-task learning.
\newblock In {\em ICML}, 2020.

\bibitem{guo2021deep}
Pengxin Guo, Chang Deng, Linjie Xu, Xiaonan Huang, and Yu~Zhang.
\newblock Deep multi-task augmented feature learning via hierarchical graph
  neural network.
\newblock In {\em ECML PKDD}, 2021.

\bibitem{han2017heterogeneous}
Hu~Han, Anil~K Jain, Fang Wang, Shiguang Shan, and Xilin Chen.
\newblock Heterogeneous face attribute estimation: A deep multi-task learning
  approach.
\newblock {\em IEEE TPAMI}, 2017.

\bibitem{javaloy2022rotograd}
Adri{\'a}n Javaloy and Isabel Valera.
\newblock Rotograd: Gradient homogenization in multitask learning.
\newblock In {\em ICLR}, 2022.

\bibitem{kendall2018multi}
Alex Kendall, Yarin Gal, and Roberto Cipolla.
\newblock Multi-task learning using uncertainty to weigh losses for scene
  geometry and semantics.
\newblock In {\em CVPR}, 2018.

\bibitem{kingma2014adam}
Diederik~P Kingma and Jimmy Ba.
\newblock Adam: A method for stochastic optimization.
\newblock {\em ICLR}, 2015.

\bibitem{kumar2012learning}
Abhishek Kumar and Hal Daume~III.
\newblock Learning task grouping and overlap in multi-task learning.
\newblock {\em ICML}, 2012.

\bibitem{ledoux2013probability}
Michel Ledoux and Michel Talagrand.
\newblock {\em Probability in Banach Spaces: isoperimetry and processes}.
\newblock 2013.

\bibitem{lee2016asymmetric}
Giwoong Lee, Eunho Yang, and Sung Hwang.
\newblock Asymmetric multi-task learning based on task relatedness and loss.
\newblock In {\em ICML}, 2016.

\bibitem{li2019towards}
Yu-Feng Li, Lan-Zhe Guo, and Zhi-Hua Zhou.
\newblock Towards safe weakly supervised learning.
\newblock {\em IEEE TPAMI}, 2019.

\bibitem{li2014towards}
Yu-Feng Li and Zhi-Hua Zhou.
\newblock Towards making unlabeled data never hurt.
\newblock {\em IEEE TPAMI}, 2014.

\bibitem{EASMTL}
Jason~Zhi Liang, Elliot Meyerson, and Risto Miikkulainen.
\newblock Evolutionary architecture search for deep multitask networks.
\newblock In Hern{\'{a}}n~E. Aguirre and Keiki Takadama, editors, {\em GECCO},
  2018.

\bibitem{liu2016hierarchical}
An-An Liu, Yu-Ting Su, Wei-Zhi Nie, and Mohan Kankanhalli.
\newblock Hierarchical clustering multi-task learning for joint human action
  grouping and recognition.
\newblock {\em IEEE TPAMI}, 2016.

\bibitem{CAGrad}
Bo~Liu, Xingchao Liu, Xiaojie Jin, Peter Stone, and Qiang Liu.
\newblock Conflict-averse gradient descent for multi-task learning.
\newblock {\em NIPS}, 2021.

\bibitem{DARTS}
Hanxiao Liu, Karen Simonyan, and Yiming Yang.
\newblock {DARTS:} differentiable architecture search.
\newblock In {\em ICLR}, 2019.

\bibitem{liu2019end}
Shikun Liu, Edward Johns, and Andrew~J Davison.
\newblock End-to-end multi-task learning with attention.
\newblock In {\em CVPR}, 2019.

\bibitem{lmzcl15}
Wu~Liu, Tao Mei, Yongdong Zhang, Cherry Che, and Jiebo Luo.
\newblock Multi-task deep visual-semantic embedding for video thumbnail
  selection.
\newblock In {\em CVPR}, 2015.

\bibitem{lu2017fully}
Yongxi Lu, Abhishek Kumar, Shuangfei Zhai, Yu~Cheng, Tara Javidi, and Rogerio
  Feris.
\newblock Fully-adaptive feature sharing in multi-task networks with
  applications in person attribute classification.
\newblock In {\em CVPR}, 2017.

\bibitem{ma2018modeling}
Jiaqi Ma, Zhe Zhao, Xinyang Yi, Jilin Chen, Lichan Hong, and Ed~H Chi.
\newblock Modeling task relationships in multi-task learning with multi-gate
  mixture-of-experts.
\newblock In {\em KDD}, 2018.

\bibitem{ManinisRK19}
Kevis{-}Kokitsi Maninis, Ilija Radosavovic, and Iasonas Kokkinos.
\newblock Attentive single-tasking of multiple tasks.
\newblock In {\em CVPR}, 2019.

\bibitem{maurer2016chain}
Andreas Maurer.
\newblock A chain rule for the expected suprema of gaussian processes.
\newblock {\em Theoretical Computer Science}, 2016.

\bibitem{maurer2016benefit}
Andreas Maurer, Massimiliano Pontil, and Bernardino Romera-Paredes.
\newblock The benefit of multitask representation learning.
\newblock {\em JMLR}, 2016.

\bibitem{misra2016cross}
Ishan Misra, Abhinav Shrivastava, Abhinav Gupta, and Martial Hebert.
\newblock Cross-stitch networks for multi-task learning.
\newblock In {\em CVPR}, 2016.

\bibitem{mohri2018foundations}
Mehryar Mohri, Afshin Rostamizadeh, and Ameet Talwalkar.
\newblock {\em Foundations of machine learning}.
\newblock MIT press, 2018.

\bibitem{MottaghiCLCLFUY14}
Roozbeh Mottaghi, Xianjie Chen, Xiaobai Liu, Nam{-}Gyu Cho, Seong{-}Whan Lee,
  Sanja Fidler, Raquel Urtasun, and Alan~L. Yuille.
\newblock The role of context for object detection and semantic segmentation in
  the wild.
\newblock In {\em CVPR}, 2014.

\bibitem{MaxRoam}
Lucas Pascal, Pietro Michiardi, Xavier Bost, Benoit Huet, and Maria Zuluaga.
\newblock Maximum roaming multi-task learning.
\newblock In {\em AAAI}, 2021.

\bibitem{rosenbaum2018routing}
Clemens Rosenbaum, Tim Klinger, and Matthew Riemer.
\newblock Routing networks: Adaptive selection of non-linear functions for
  multi-task learning.
\newblock In {\em ICLR}, 2018.

\bibitem{ruder2019latent}
Sebastian Ruder, Joachim Bingel, Isabelle Augenstein, and Anders S{\o}gaard.
\newblock Latent multi-task architecture learning.
\newblock In {\em AAAI}, 2019.

\bibitem{silberman2012indoor}
Nathan Silberman, Derek Hoiem, Pushmeet Kohli, and Rob Fergus.
\newblock Indoor segmentation and support inference from rgbd images.
\newblock In {\em ECCV}, 2012.

\bibitem{standley2020tasks}
Trevor Standley, Amir Zamir, Dawn Chen, Leonidas Guibas, Jitendra Malik, and
  Silvio Savarese.
\newblock Which tasks should be learned together in multi-task learning?
\newblock In {\em ICML}, 2020.

\bibitem{strezoski2019many}
Gjorgji Strezoski, Nanne~van Noord, and Marcel Worring.
\newblock Many task learning with task routing.
\newblock In {\em ICCV}, 2019.

\bibitem{sun2021task}
Guolei Sun, Thomas Probst, Danda~Pani Paudel, Nikola Popovi{\'c}, Menelaos
  Kanakis, Jagruti Patel, Dengxin Dai, and Luc Van~Gool.
\newblock Task switching network for multi-task learning.
\newblock In {\em ICCV}, 2021.

\bibitem{SunPFS20}
Ximeng Sun, Rameswar Panda, Rog{\'{e}}rio Feris, and Kate Saenko.
\newblock Adashare: Learning what to share for efficient deep multi-task
  learning.
\newblock In {\em NIPS}, 2020.

\bibitem{tang2020progressive}
Hongyan Tang, Junning Liu, Ming Zhao, and Xudong Gong.
\newblock Progressive layered extraction (ple): A novel multi-task learning
  (mtl) model for personalized recommendations.
\newblock In {\em RecSys}, 2020.

\bibitem{tang2022deepsafe}
Huayi Tang and Yong Liu.
\newblock Deep safe incomplete multi-view clustering: Theorem and algorithm.
\newblock In {\em ICML}, 2022.

\bibitem{tang2022deep}
Huayi Tang and Yong Liu.
\newblock Deep safe multi-view clustering: Reducing the risk of clustering
  performance degradation caused by view increase.
\newblock In {\em CVPR}, 2022.

\bibitem{tao2018reliable}
Hong Tao, Chenping Hou, Xinwang Liu, Tongliang Liu, Dongyun Yi, and Jubo Zhu.
\newblock Reliable multi-view clustering.
\newblock In {\em AAAI}, 2018.

\bibitem{vandenhende2021multi}
Simon Vandenhende, Stamatios Georgoulis, Wouter Van~Gansbeke, Marc Proesmans,
  Dengxin Dai, and Luc Van~Gool.
\newblock Multi-task learning for dense prediction tasks: A survey.
\newblock {\em IEEE TPAMI}, 2021.

\bibitem{wang2019characterizing}
Zirui Wang, Zihang Dai, Barnab{\'a}s P{\'o}czos, and Jaime Carbonell.
\newblock Characterizing and avoiding negative transfer.
\newblock In {\em CVPR}, 2019.

\bibitem{yang2020transfer}
Qiang Yang, Yu~Zhang, Wenyuan Dai, and Sinno~Jialin Pan.
\newblock {\em Transfer learning}.
\newblock 2020.

\bibitem{PCGrad}
Tianhe Yu, Saurabh Kumar, Abhishek Gupta, Sergey Levine, Karol Hausman, and
  Chelsea Finn.
\newblock Gradient surgery for multi-task learning.
\newblock {\em NIPS}, 2020.

\bibitem{zamir2018taskonomy}
Amir~R Zamir, Alexander Sax, William Shen, Leonidas~J Guibas, Jitendra Malik,
  and Silvio Savarese.
\newblock Taskonomy: Disentangling task transfer learning.
\newblock In {\em CVPR}, 2018.

\bibitem{zhang2021multi}
Yi~Zhang, Yu~Zhang, and Wei Wang.
\newblock Multi-task learning via generalized tensor trace norm.
\newblock In {\em KDD}, 2021.

\bibitem{zy21}
Yu~Zhang and Qiang Yang.
\newblock A survey on multi-task learning.
\newblock {\em IEEE TKDE}, 2021.

\bibitem{zhang2010convex}
Yu~Zhang and Dit-Yan Yeung.
\newblock A convex formulation for learning task relationships in multi-task
  learning.
\newblock In {\em UAI}, 2010.

\bibitem{zhang2010multi}
Yu~Zhang and Dit-Yan Yeung.
\newblock Multi-task warped gaussian process for personalized age estimation.
\newblock In {\em CVPR}, 2010.

\end{thebibliography}
\appendices

\section{Proofs}

In this section, we provide proofs for all the theorems.

\subsection{Generalization Bound for Problem \eqref{eqa: obj}}

To help analyze the generalization bound of the DSMTL method, we first introduce a useful theorem in terms of Gaussian averages \cite{bartlett2002rademacher,maurer2016benefit}.
\begin{theorem}\label{thmB:1}
Let $\mathcal{G}$ be a class of functions $\vartheta: \mathcal{X} \to [0,1]^\top$, and $\mu_1,...,\mu_m$ be the probability measure on $\mathcal{X}$ with $\bar{\mathbf{X}} = (\mathbf{X}_1,...,\mathbf{X}_m) \sim \prod_{t=1}^m(\mu_t)^n$, where $\mathbf{X}_t= (\mathbf{x}_{t}^{1},\ldots,\mathbf{x}_{t}^{n})$. Let  and $\gamma$ be a vector of independent standard normal variables and $Z$ be the random set $\{(\vartheta_t(\mathbf{x}_{t}^i)): \vartheta \in \mathcal{G} \}$ where $\vartheta_t$ are functions chosen from hypothesis class $\mathcal{G}$. Then for all $\vartheta\in \mathcal{G}$, with probability at least $1-\delta$, we have
{\small
\begin{equation*}
\frac{1}{m} \sum_t \left(\mathbb{E}_{\mathbf{x}\sim \mu_t}[\vartheta_t(\mathbf{x})] - \frac{1}{n} \sum_{i}\vartheta_t(\mathbf{x}_{t}^i)   \right)
    \le \frac{\sqrt{2\pi}G(Z)}{mn} + \sqrt{\frac{9\ln \frac{2}{\delta}}{2mn}},
\end{equation*}
}\noindent
where $G(Z) = \mathbb{E}[\sup_{z\in Z}\left< \gamma,z \right>]$ is the Gaussian average of the random set $Z$.
\end{theorem}

Based on Theorem \ref{thmB:1}, we first establish the following uniform bound for problem \eqref{eqa: obj}.
\begin{theorem}\label{thmB:2}
Suppose Assumption \ref{assumption:1} is satisfied. Then for $(\bar{\mathbf{X}},\bar{\mathbf{Y}}) \sim \prod_{t=1}^m(\mu_t)^n$, with probability at least $1-\delta$, we have
\begin{equation}\label{thmB:1eq}
\begin{aligned}
\mathcal{E}- &\frac{1}{mn}\sum_{t,i} \mathcal{L}_t(\mathbf{y}^i_t,h_t( g_t(f_S(\mathbf{x}^i_t),f_t(\mathbf{x}^i_t)))) \\&
\le  \frac{C_1 G(\mathcal{F}(\bar{\mathbf{X}}))}{mn}
+ \frac{C_2Q}{\sqrt{n}}  + \sqrt{\frac{9\ln \frac{2}{\delta}}{2mn}},
\end{aligned}
\end{equation}
where $C_1$, $C_2$ are two constants, the quantity $Q$ is defined as
\begin{equation*}
Q = \sup_{z\not = \tilde{z} \in \mathbb{R}^{nq}} \frac{1}{\|z-\tilde{z}\|} \mathbb{E} \sup_{h_t\in \mathcal{H}}\sum_{i=1}^n  \gamma_i (h_t(z_{i}) - h_t(\tilde{z}_{i})),
\end{equation*}
and $\gamma$ is a vector of independent standard normal variables.
\end{theorem}

\begin{proof}
According to Theorem \ref{thmB:1}, for $h_t\in\mathcal{H}$, $f_t,f_S \in \mathcal{F}$, and $\alpha_t \in \mathcal{M}$, with probability at least $1-\delta$, we have
\begin{align*}
    \mathcal{E}- &\frac{1}{mn}\sum_{t,i} \mathcal{L}_t(\mathbf{y}^i_t,h_t( g_t(f_S(\mathbf{x}^i_t),f_t(\mathbf{x}^i_t)))) \\&
\le  \frac{\sqrt{2\pi} G(\mathcal{S})}{mn}
 + \sqrt{\frac{9\ln \frac{2}{\delta}}{2mn}},
\end{align*}
where $S = \{\mathcal{L}_t(\mathbf{y}^i_t,h_t( g_t(f_S(\mathbf{x}^i_t),f_t(\mathbf{x}^i_t))))\} \subseteq \mathbb{R}^{mn}$ and $G(S)$ represents the Gaussian average of the set $S$. Then by using the Lipschitz property of $\mathcal{L}_t$ and Slepian's Lemma \cite{ledoux2013probability}, we have $G(S)\le G(S')$, where $S' = \{  h_t( g_t(f_S(\mathbf{x}^i_t),f_t(\mathbf{x}^i_t))) : h_t\in \mathcal{H}, \alpha_t \in \mathcal{M}, f_t,f_S \in \mathcal{F}\}$.

Note that the input data $\bar{\mathbf{X}}\in \mathcal{X}^{mn}$ and the encoders $f_1,\ldots,f_m,f_S: \mathcal{X}\to \mathbb{R}^{q}$ are mapping functions chosen from $\mathcal{F}$, the random set $\mathcal{K}(\bar{\mathbf{X}}) \subseteq \mathbb{R}^{mnq}$ is defined as  $\mathcal{K}(\bar{\mathbf{X}}) = \{ g_t(f_S(\mathbf{x}^i_t),f_t(\mathbf{x}^i_t)): \alpha_t \in \mathcal{M}, f_t,f_S \in \mathcal{F}\}$. We define a class of functions $\mathcal{H}' = \{z \in \mathbb{R}^{mnq}\mapsto h_t(z_t^i): h_t\in \mathcal{H}\}$. Therefore, we have  $\mathcal{S}'(\bar{\mathbf{X}}) = \mathcal{H}'(\mathcal{K}(\bar{\mathbf{X}}))$.

By using Theorem 2 in \cite{maurer2016chain}, we obtain
\begin{align*}
    G(S') \le  c_1 & L(\mathcal{H}')G(\mathcal{K}(\bar{\mathbf{X}})) +c_2D(\mathcal{K}(\bar{\mathbf{X}}))Q(\mathcal{H}') \\
    &+ \min_{z\in \mathcal{K}(\bar{\mathbf{X}})}G(\mathcal{H}'(z)),
\end{align*}
where $c_1$ and $c_2$ are two constants, $L(\mathcal{H}')$ denotes the Lipschitz constant of the functions in $\mathcal{H}'$, $D(\mathcal{K}(\bar{\mathbf{X}})) = 2 \sup_{\varphi \in \mathcal{K}}\|\varphi(\bar{\mathbf{X}})\|$ denotes the Euclidean diameter of the set $\mathcal{F}(\bar{\mathbf{X}})$, and
{\small
\begin{equation*}
    Q(\mathcal{H}') =\sup_{z\not = \tilde{z} \in \mathbb{R}^{mnq}} \frac{1}{\|z-\tilde{z}\|} \mathbb{E}\sup_{\psi\in\mathcal{H}'} \left< \gamma, \psi(z) - \psi(\tilde{z})\right>.
\end{equation*}
}\noindent

Let $z,\tilde{z} \in \mathbb{R}^{mnq}$, where $z = (z_{t}^i)$, $\tilde{z} = (\tilde{z}_{t}^i)$ and $z_{t}^i, \tilde{z}_{t}^i \in \mathbb{R}^{q}$. Then for any functions $\psi \in \mathcal{H}'$,  we have
{\small
\begin{align*}
    & \mathbb{E}\sup_{\psi\in\mathcal{H}'} \left< \gamma, \psi(z) - \psi(\tilde{z})\right>  \\
    = & \sum_{t=1}^m\mathbb{E} \sup_{h\in \mathcal{H}} \sum_{i=1}^n \gamma_i (h(z_t^i) - h(\tilde{z}_t^i)) \\
    \le & \sqrt{m} \left(\sum_{t=1}^m (  \mathbb{E} \sup_{h\in \mathcal{H}} \sum_{i=1}^n \gamma_i (h(z_t^i) - h(\tilde{z}_t^i))   )^{2}\right)^{1/2} \\
    \le & \sqrt{m} \left(\sum_{t=1}^m Q^2\sum_{i=1}^n \|z_t^i - \tilde{z}_t^i\|^2 \right)^{1/2} \nonumber \\
    = & \sqrt{m} Q \|z-\tilde{z}\|,
\end{align*}
}\noindent
 where the first inequality is due to the Cauchy-Schwarz inequality and the second inequality is due to the inequality $\|z-\tilde{z}\| \le \sup \|z-\tilde{z}\|$ holds for all $z\not = \tilde{z} \in \mathbb{R}^{nq}$.
Therefore, $Q(\mathcal{H}') \le \sqrt{m} Q$. Moreover, suppose the functions in hypothesis classes $\mathcal{H}$ are $M$-Lipschitz continuous, we have
\begin{equation*}\label{sub:5}
    \| \psi(z)-\psi(\tilde{z}) \|^2  = \sum_{t,i}(h_t(z_t^i) - h_t(\tilde{z}_{t}^{i}))^2 \le  M^2 \|z -\tilde{z}\|^2
\end{equation*}
where the inequality is due to the Lipschitz property. So we obtain $L(\mathcal{H}') \le M$. Since $0 \in \mathcal{F}$, we have $0 \in \mathcal{K}(\bar{\mathbf{X}})$. Note that $h(0)=0$, and hence $\min_{z\in \mathcal{F}(\bar{\mathbf{X}})}G(\mathcal{H}'(z)) = 0$ by setting $z=0$. Therefore, we have
\begin{equation}\label{sub:2}
    G(S) \le c_1 M G(\mathcal{K}(\bar{\mathbf{X}})) +2 c_2 \sqrt{m}Q\sup_{\varphi \in \mathcal{K}}\|\varphi(\bar{\mathbf{X}})\|.
\end{equation}

Recall that the random set $\mathcal{F}(\bar{\mathbf{X}}) \subseteq \mathbb{R}^{mnq}$ is defined as  $\mathcal{F}(\bar{\mathbf{X}}) = \{ (f_t(\mathbf{x}^i_t)): f_t \in \mathcal{F}\}$. Denote the $j$th entry of the vectors $f_S(\mathbf{x}^i_t)$ and $f_t(\mathbf{x}^i_t)$ by $f_{S,j}(\mathbf{x}^i_t)$ and $f_{t,j}(\mathbf{x}^i_t)$, respectively, and let $\gamma_{j,t,i}$ be the corresponding independent standard normal variable. Then we have
\begin{align*}
    G(\mathcal{K}(\bar{\mathbf{X}}))&= \mathbb{E}\left[\sup_{f_t,f_S,\alpha_t} \sum_{j=1}^q\sum_{t,i}\gamma_{j,t,i}g_t(f_{S,j}(\mathbf{x}^i_t),f_{t,j}(\mathbf{x}^i_t))\Big| \mathbf{x}^i_t\right] \\
    & = \mathbb{E}\left[\sup_{f_t\in \mathcal{F}} \sum_{j=1}^q\sum_{t,i}\gamma_{j,t,i}f_{t,j}(\mathbf{x}^i_t))\Big| \mathbf{x}^i_t\right] = G(\mathcal{F}(\bar{\mathbf{X}})),
\end{align*}
where the first and third equality are due to the definition of the Gaussian average, and the second equality holds since $f_t$ and $f_S$ are chosen from the same class $\mathcal{F}$ and $\mathcal{F}$ is uniformly bounded.

Since the hypothesis classes $\mathcal{F}$ is uniformly bounded, suppose that $\|f(\mathbf{x}^i_t)\| \le R$ for all $f\in \mathcal{F}$ and we have
\begin{align*}
    \sup_{\varphi \in \mathcal{K}}\|\varphi(\bar{\mathbf{X}})\| &= \sqrt{mn}\sup_{f_t,f_S,\alpha_t}\max_{i,t} \| g_t(f_S(\mathbf{x}^i_t),f_t(\mathbf{x}^i_t))\| \\
    & \le \sqrt{mn} \sup_{f\in\mathcal{F}}\max_{i,t} \| f(\mathbf{x}^i_t)\| \le \sqrt{mn} R,
\end{align*}
where the first inequality holds since $f_t$ and $f_S$ are chosen from the same class $\mathcal{F}$, and the second inequality holds since $\|f(\mathbf{x}^i_t)\| \le R$. Therefore, we get
\begin{align*}
    \mathcal{E}- &\frac{1}{mn}\sum_{t,i} \mathcal{L}_t(\mathbf{y}^i_t,h_t( g_t(f_S(\mathbf{x}^i_t),f_t(\mathbf{x}^i_t)))) \\&
\le  \frac{c_1M \sqrt{2\pi} G(\mathcal{F}(\bar{\mathbf{X}}))}{mn}
+ \frac{2c_2R\sqrt{2\pi}Q}{\sqrt{n}}
 + \sqrt{\frac{9\ln \frac{2}{\delta}}{2mn}}.
\end{align*}
By setting $C_1 = c_1M \sqrt{2\pi}$ and $C_2 = 2c_2R\sqrt{2\pi}$, we reach the conclusion.
\end{proof}

\begin{remark}
According to the McDiarmid’s inequality \cite{mohri2018foundations}, the upper bound in Theorem \ref{thmB:1} also holds for the sum of $\vartheta_t(\mathbf{x}_{t}^i)$ minus its expectation. Thus following the same proof as above, we can verify that the following inequality holds under the same assumption in Theorem \ref{thmB:1}:
\begin{equation}\label{thmB:2eq}
\begin{aligned}
\frac{1}{mn}&\sum_{t,i} \mathcal{L}_t(\mathbf{y}^i_t,h_t( g_t(f_S(\mathbf{x}^i_t),f_t(\mathbf{x}^i_t)))) - \mathcal{E}\\&
\le  \frac{C_1 G(\mathcal{F}(\bar{\mathbf{X}}))}{mn}
+ \frac{C_2Q}{\sqrt{n}}  + \sqrt{\frac{9\ln \frac{2}{\delta}}{2mn}}.
\end{aligned}
\end{equation}
\end{remark}

\subsection{Proof of Theorem \ref{thm1}} \label{app:proof_theorem_1}
\begin{proof}
The STL model aims to solve the following optimization problem as
\begin{equation*}
    \min_{h^{\text{STL}}_t \in \mathcal{H},f^{\text{STL}}_t\in \mathcal{F}} \frac{1}{mn}\sum_{i=1}^n \sum_{t=1}^m \mathcal{L}_t(\mathbf{y}^i_t,h^{\text{STL}}_t( f^{\text{STL}}_t(\mathbf{x}_t^i))),
\end{equation*}
where its solution is denoted by $\{\hat{h}^{\text{STL}}_t\}$ and $\{\hat{f}^{\text{STL}}_t\}$. Therefore, the minimal empirical loss of the STL model on task $t$ is computed as
\begin{equation*}
    \hat{L}^{\text{STL}}_t = \frac{1}{n}\sum_{i=1}^n  \mathcal{L}_t(\mathbf{y}^i_t,\hat{h}^{\text{STL}}_t(\hat{f}^{\text{STL}}_t(\mathbf{x}_t^i))).
\end{equation*}
For the DSMTL-IL model, we set $\alpha_t = 0$ for all tasks in the first training stage, thus the corresponding objective function is the same as that of the STL model. Therefore, the solution of the first stage in the DSMTL-IL model, i.e., $\{\hat{h}_t\}$ and $\{\hat{f}_t\}$, satisfies $\hat{h}_t = \hat{h}^{\text{STL}}_t$ and $\hat{f}_t= \hat{f}^{\text{STL}}_t$. In the second training stage of the DSMTL-IL model, for any public encoder $f_S$ in task $t$, we have
\begin{align*}
    \min_{\alpha_t\in \mathcal{M}} & \frac{1}{n}\sum_{i=1}^n  \mathcal{L}_t(\mathbf{y}^i_t,\hat{h}_t( \alpha_tf_S(\mathbf{x}_t^i)+(1-\alpha_t)\hat{f}_t(\mathbf{x}_t^i))) \\
    & \le \frac{1}{n}\sum_{i=1}^n  \mathcal{L}_t(\mathbf{y}^i_t,\hat{h}^{\text{STL}}_t(\hat{f}^{\text{STL}}_t(\mathbf{x}_t^i))) = \hat{L}^{\text{STL}}_t.
\end{align*}
Therefore, $L^*_t \le \hat{L}^{\text{STL}}_t \le L^{\text{STL}}_t$ holds for all $1\le t \le m$. This finishes the proof.
\end{proof}

\subsection{Proof of Theorem \ref{thm2}} \label{app:proof_theorem_2}
\begin{proof}

According to Theorem \ref{thmB:2}, with $m$ tasks and data $\bar{\mathbf{X}}$, we have following inequality
\begin{equation}
\begin{aligned}
\frac{1}{m}\sum_{t=1}^m\mathcal{E}_t- &\frac{1}{mn}\sum_{t,i} \mathcal{L}_t(\mathbf{y}^i_t,h_t( g_t(f_S(\mathbf{x}^i_t),f_t(\mathbf{x}^i_t)))) \\&
\le  \frac{c_1 G(\mathcal{F}(\bar{\mathbf{X}}))}{mn}
+ \frac{c_2Q}{\sqrt{n}}  + \sqrt{\frac{9\ln \frac{2}{\delta}}{2mn}},
\end{aligned}
\end{equation}
holds with probability at least $1-\delta$. Thus, for one single task such as task $t$ and its corresponding data $\mathbf{X}_t$, setting the number of tasks in the above inequality to be $1$ gives
\begin{equation}
\begin{aligned}\label{thm2:proof2}
    \mathcal{E}_t - &\frac{1}{n}\sum_i \mathcal{L}_t(\mathbf{y}^i_t,h_t( g_t(f_S(\mathbf{x}^i_t),f_t(\mathbf{x}^i_t)))) \\&
    \le \frac{c_1 G(\mathcal{F}'(\mathbf{X}_t))}{n} +\frac{c_2Q}{\sqrt{n}} + \sqrt{\frac{9\ln{\frac{2}{\delta}}}{2n}}.
\end{aligned}
\end{equation}
By substituting the solution of problem \eqref{eqa: obj_al} into inequality \eqref{thm2:proof2}, we have following inequality as
\begin{equation}\label{proof2:eq1}
     \hat{\mathcal{E}}_t -L^*_t  \le \frac{c_1 G(\mathcal{F}'(\mathbf{X}_t))}{n} +\frac{ c_2Q}{\sqrt{n}} + \sqrt{\frac{9\ln{\frac{2}{\delta}}}{2n}}.
\end{equation}

Moreover, since the STL model can be considered as a special case of the DSMTL-IL model where $g_t$ adopts $g_t^0$ defined in Eq. \eqref{eq:stl_smtl}, substituting it into the inequality \eqref{thmB:2eq} gives
\begin{equation}\label{proof2:eq2}
 L^{\text{STL}} - \mathcal{E}^{\text{STL}}\le  \frac{c_1' G(\mathcal{F}(\bar{\mathbf{X}}))}{nm} + \frac{c_2'Q}{\sqrt{n}}  + \sqrt{\frac{9\ln \frac{2}{\delta}}{2mn}}.
\end{equation}
 Therefore, for the STL model in task $t$, setting $m$ to 1 in the above inequality gives
\begin{equation}\label{proof2:eq3}
     L_t^{\text{STL}} - \mathcal{E}^{\text{STL}}_t \le \frac{c_1' G(\mathcal{F}'(\mathbf{X}_t))}{n} +\frac{c_2'Q}{\sqrt{n}} + \sqrt{\frac{9\ln{\frac{2}{\delta}}}{2n}}.
\end{equation}

Add the inequalities \eqref{proof2:eq1} and \eqref{proof2:eq3} we get
\begin{equation*}
\hat{\mathcal{E}}_t - L^*_t + L_t^{\text{STL}} - \mathcal{E}^{\text{STL}}_t \le \frac{C_1 G(\mathcal{F}'(\mathbf{X}_t))}{n} +\frac{ C_2Q}{\sqrt{n}} + \sqrt{\frac{18\ln{\frac{2}{\delta}}}{n}},
\end{equation*}
where $C_1$ and $C_2$ are two constants and $\mathcal{F}'(\bar{\mathbf{X}}_t) = \{ f(\mathbf{x}^i_t): f \in \mathcal{F}\} \subseteq \mathbb{R}^{nq}$. According to Theorem \ref{thm1}, there exists a constant $\varepsilon_t \ge 0$ such that $\hat{L}^*_t + \varepsilon_t = L_t^{\text{STL}}$. Therefore, we have
\begin{equation*}
    \hat{\mathcal{E}}_t+\epsilon_t \le \mathcal{E}^{\text{STL}}_t + \frac{C_1 G(\mathcal{F}'(\mathbf{X}_t))}{n} +\frac{ C_2Q}{\sqrt{n}}+\sqrt{\frac{18\ln{\frac{2}{\delta}}}{n}},
\end{equation*}
which completes the proof.
\end{proof}

\subsection{Proof of Theorem \ref{thm3}} \label{app:proof_theorem_3}
\begin{proof}
It is easy to see that the STL model can be considered as a special case of the DSMTL-JL model when $\alpha_t = 0$ holds for $1\le t \le m$ and the HPS model is also a special case of the DSMTL-JL model when $\alpha_t = 1$ holds for $1\le t \le m$.

The empirical loss of the DSMTL model is formulated as
$$L = \frac{1}{mn}\sum_{i=1}^n \sum_{t=1}^m \mathcal{L}_t(\mathbf{y}^i_t,h_t( g_t(f_S(\mathbf{x}^i_t),f_t(\mathbf{x}^i_t)))).$$
We have $L^{\text{STL}} = L(\Theta)\mid_{\alpha_t = 0}$ and $L^{\text{HPS}} = L(\Theta)\mid_{\alpha_t = 1}$. Since $L^* \le L$, we can get $L^* \le \min\{L^{\text{STL}},L^{\text{HPS}}\}$ and hence we reach the conclusion.
\end{proof}

\subsection{Proof of Theorem \ref{thm4}} \label{app:proof_theorem_4}
\begin{proof}
Let $L^*$ be the optimal value of problem \eqref{eqa: obj}. Substituting the solution of problem \eqref{eqa: obj} into inequality \eqref{thmB:1eq} gives
\begin{equation}\label{proof4:eq1}
\hat{\mathcal{E}} -  L^*
\le  \frac{c_1 G(\mathcal{F}(\bar{\mathbf{X}}))}{mn}
+ \frac{c_2Q}{\sqrt{n}}  + \sqrt{\frac{9\ln \frac{2}{\delta}}{2mn}}.
\end{equation}
Based on inequalities \eqref{proof4:eq1} and \eqref{proof2:eq2}, with
probability $1-\delta$, we have
\begin{equation*}
\hat{\mathcal{E}} -L^* + L^{\text{STL}}- \mathcal{E}^{\text{STL}}  \le  \frac{C_1 G(\mathcal{F}(\bar{\mathbf{X}}))}{nm}
+ \frac{C_2Q}{\sqrt{n}}  + \sqrt{\frac{18\ln \frac{2}{\delta}}{mn}},
\end{equation*}
where $C_1$ and $C_2$ are two constants. According to Theorem \ref{thm3}, there exists a constant $\varepsilon \ge 0$ such that $L^* + \varepsilon = L^{\text{STL}}$. Therefore, we have
\begin{equation*}
    \hat{\mathcal{E}}+\epsilon \le \mathcal{E}^{\text{STL}} + \frac{C_1 G(\mathcal{F}(\bar{\mathbf{X}}))}{nm}
+ \frac{C_2Q}{\sqrt{n}} +\sqrt{\frac{18\ln{\frac{2}{\delta}}}{mn}},
\end{equation*}
where we reach the conclusion.
\end{proof}

\subsection{Proof of Theorem \ref{thm5}} \label{app:proof_theorem_5}
\begin{proof}
Let $\{f_t\}^*$, $f_S^*$, $\{h_t\}^*$, $\{g_t\}^*$ be the minimizer in $\mathcal{E}^*$. We can decompose $\hat{\mathcal{E}}- \mathcal{E}^*$ as
{\small
\begin{align*}
    \hat{\mathcal{E}}- \mathcal{E}^* =& \Big( \hat{\mathcal{E}} -  \frac{1}{mn}\sum_{t,i}\mathcal{L}_t(\mathbf{y}_t^i, h_t(g_t(f_S(\mathbf{x}_t^i), f_t(\mathbf{x}_t^i)) \Big) \\
    & +  \Big( \frac{1}{mn}\sum_{t,i} \mathcal{L}_t(\mathbf{y}_t^i, h_t(g_t(f_S(\mathbf{x}_t^i), f_t(\mathbf{x}_t^i))    \\
    & \qquad- \frac{1}{mn}\sum_{t,i}\mathcal{L}_t(\mathbf{y}_t^i, h^*_t(g^*_t(f^*_S(\mathbf{x}_t^i), f^*_t(\mathbf{x}_t^i))  \Big)  \\
    & + \Big( \frac{1}{mn}\sum_{t,i}\mathcal{L}_t(\mathbf{y}_t^i, h^*_t(g^*_t(f^*_S(\mathbf{x}_t^i), f^*_t(\mathbf{x}_t^i)) -\mathcal{E}^* \Big) \nonumber ,
\end{align*}
}\noindent
where the first term can be bounded by substituting inequality \eqref{thmB:1eq} and the last term can be regarded as $mn$ random variables $\mathcal{L}_t(\mathbf{y}_t^i, h^*_t(g^*_t(f^*_S(\mathbf{x}_t^i), f^*_t(\mathbf{x}_t^i))$ with values in $[0,1]$. By using Hoeffding's inequality, with probability at least $1-\delta$, we have
\begin{equation*}
    \frac{1}{mn}\sum_{ti}\mathcal{L}_t(\mathbf{y}_t^i , h_t^*(\omega_t^{*T}\varphi^*(\mathbf{x}_t^i)))-\mathcal{E}^* \le \sqrt{\frac{\ln \frac{1}{\delta}}{2mn}}.
\end{equation*}The second term is non-positive due to the definition of minimizers. Therefore, we have

\begin{equation*}
\hat{\mathcal{E}}- \mathcal{E}^*  \le\frac{C_1 G(\mathcal{F}(\bar{\mathbf{X}}))}{mn}
+ \frac{C_2Q}{\sqrt{n}}+ \sqrt{\frac{8\ln\frac{4}{\delta}}{mn}},
\end{equation*}
where we reach the conclusion.
\end{proof}

\end{document}